\documentclass[lettersize,journal]{IEEEtran}
\usepackage{amsmath,amsfonts}
\usepackage{algorithmic}
\usepackage{algorithm}
\usepackage{array}
\usepackage[caption=false,font=normalsize,labelfont=sf,textfont=sf]{subfig}
\usepackage{textcomp}
\usepackage{stfloats}
\usepackage{url}
\usepackage{verbatim}
\usepackage{graphicx}
\usepackage{cite}

\usepackage{xcolor}
\usepackage{xspace}
\usepackage{booktabs} 
\usepackage{multirow}
\usepackage{enumitem}
\usepackage{subcaption}
\usepackage[table]{xcolor}

\usepackage{amsthm}
\usepackage{amsmath}
\usepackage{amssymb}

\newtheorem{lemma}{Lemma}
\newtheorem{corollary}{Corollary}
\newtheorem{proposition}{Proposition}

\usepackage{tcolorbox}

\definecolor{beaublue}{rgb}{1.0, 1.0, 1.0}
\definecolor{blackish}{rgb}{0.2, 0.2, 0.2}

\definecolor{myblue}{HTML}{508FC9}

\usepackage{amsmath,amsfonts,bm}

\def\eqref#1{equation~\ref{#1}}

\def\1{\bm{1}}

\def\vf{{\bm{f}}}

\def\vp{{\bm{p}}}

\def\vs{{\bm{s}}}

\def\vx{{\bm{x}}}

\def\vz{{\bm{z}}}

\DeclareMathAlphabet{\mathsfit}{\encodingdefault}{\sfdefault}{m}{sl}
\SetMathAlphabet{\mathsfit}{bold}{\encodingdefault}{\sfdefault}{bx}{n}

\def\gD{{\mathcal{D}}}

\def\gX{{\mathcal{X}}}
\def\gY{{\mathcal{Y}}}

\makeatletter
\DeclareRobustCommand\onedot{\futurelet\@let@token\@onedot}
\def\@onedot{\ifx\@let@token.\else.\null\fi\xspace}

\def\eg{\emph{e.g}\onedot} 
\def\ie{\emph{i.e}\onedot} 
 
 \def\vs{\emph{vs}\onedot}

\makeatother

\begin{document}

\title{Trust-Aware Joint Feature-Prediction Discrepancy \\for Robust Domain Adaptation
}

\author{Xi Ding$^\dagger$, Lei Wang$^\dagger$, Syuan-Hao Li, Yongsheng Gao
\thanks{
This work is supported in part by the Australian Research Council (ARC) under Industrial Transformation Research Hub Grant IH180100002. (Corresponding author: Yongsheng Gao.)

$\dagger$ are co-first authors with equal contribution.

Xi Ding, Lei Wang, Syuan-Hao Li, and Yongsheng Gao are with the School of Engineering and Built Environment, Griffith University, Australia (emails: x.ding@griffith.edu.au; l.wang4@griffith.edu.au; syuan-hao.li@griffithuni.edu.au; yongsheng.gao@griffith.edu.au). 
}}



\maketitle

\begin{abstract}
Domain adaptation aims to mitigate performance degradation caused by distribution shifts between a labeled source domain and an unlabeled or sparsely labeled target domain. Most existing approaches estimate domain discrepancy either in feature space or in prediction space. However, these single-perspective strategies overlook a critical problem under domain shift: the reliability of the signals used for alignment. In practice, both learned representations and semantic predictions may become unreliable, and treating all target samples equally can lead to misleading alignment and suboptimal transfer.
In this work, we introduce trust-aware domain adaptation, a principled framework that models domain discrepancy through the reliability of feature and prediction signals. Central to our approach is the Joint Feature-Prediction Discrepancy (JFPD), a unified formulation that jointly captures representation divergence and prediction divergence while weighting their contributions by sample-specific trust. Trust is quantified through two complementary mechanisms: \textit{uncertainty-aware trust}, derived from prediction entropy to suppress unreliable predictions, and \textit{semantic-alignment trust}, computed from prototype similarity in feature space to emphasize well-aligned representations. By prioritizing confident and semantically consistent samples while down-weighting noisy or ambiguous ones, JFPD provides a reliability-aware estimate of domain discrepancy.
We further integrate JFPD into a training objective that guides adaptation toward trustworthy regions of the target domain. Extensive experiments on standard digit and object recognition benchmarks demonstrate that the proposed framework consistently achieves superior adaptation performance and yields discrepancy estimates that correlate with target-domain error. 
This work addresses, for the first time, the importance of modeling trust in the interaction between feature representations and predictions for robust domain adaptation.
\end{abstract}

\begin{IEEEkeywords}
Fine-tuning, trust-aware learning, cross-domain transfer, trust-aware loss.
\end{IEEEkeywords}

\begin{figure}[tbp]
\centering
\includegraphics[width=0.48\linewidth]{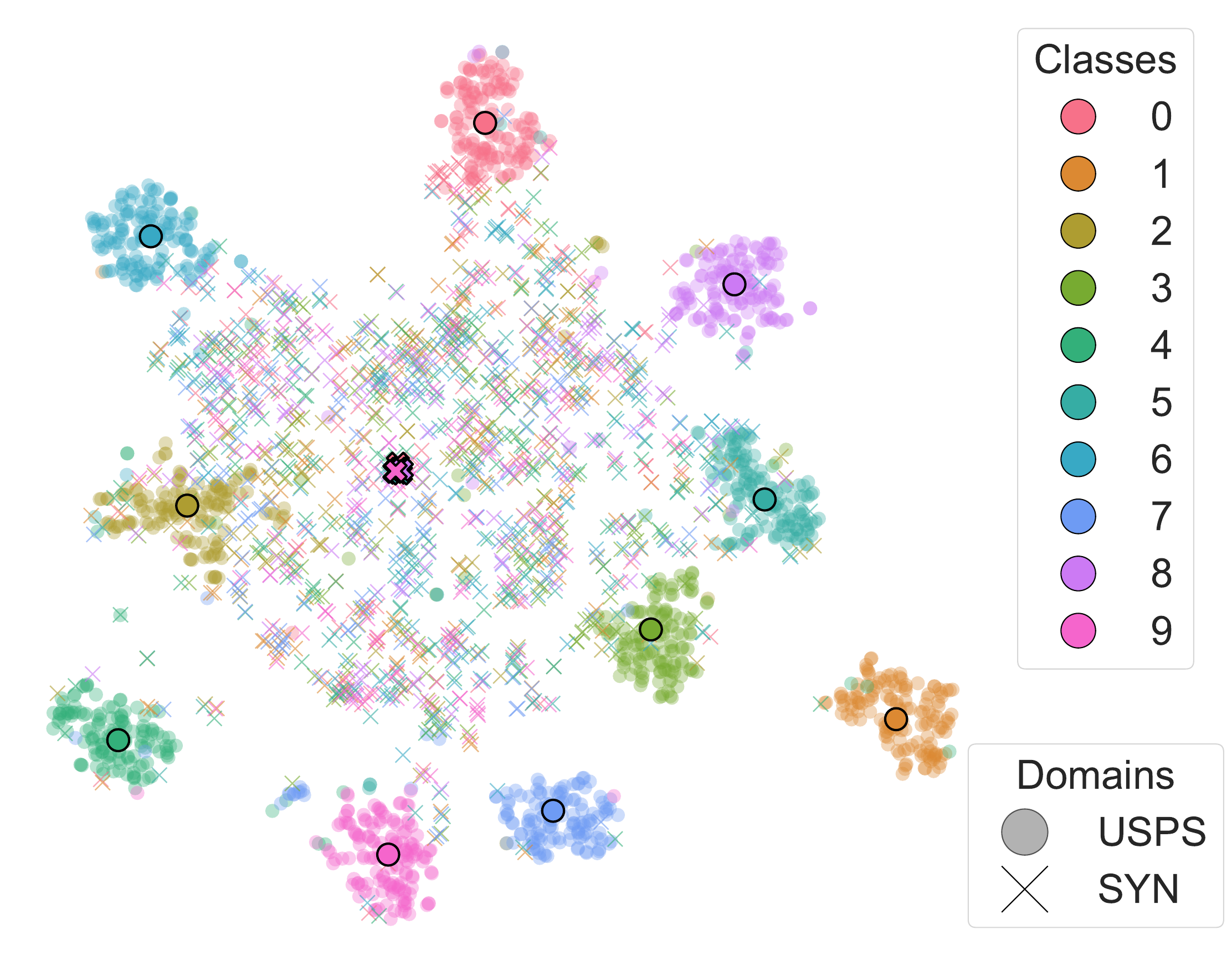}
\includegraphics[width=0.48\linewidth]{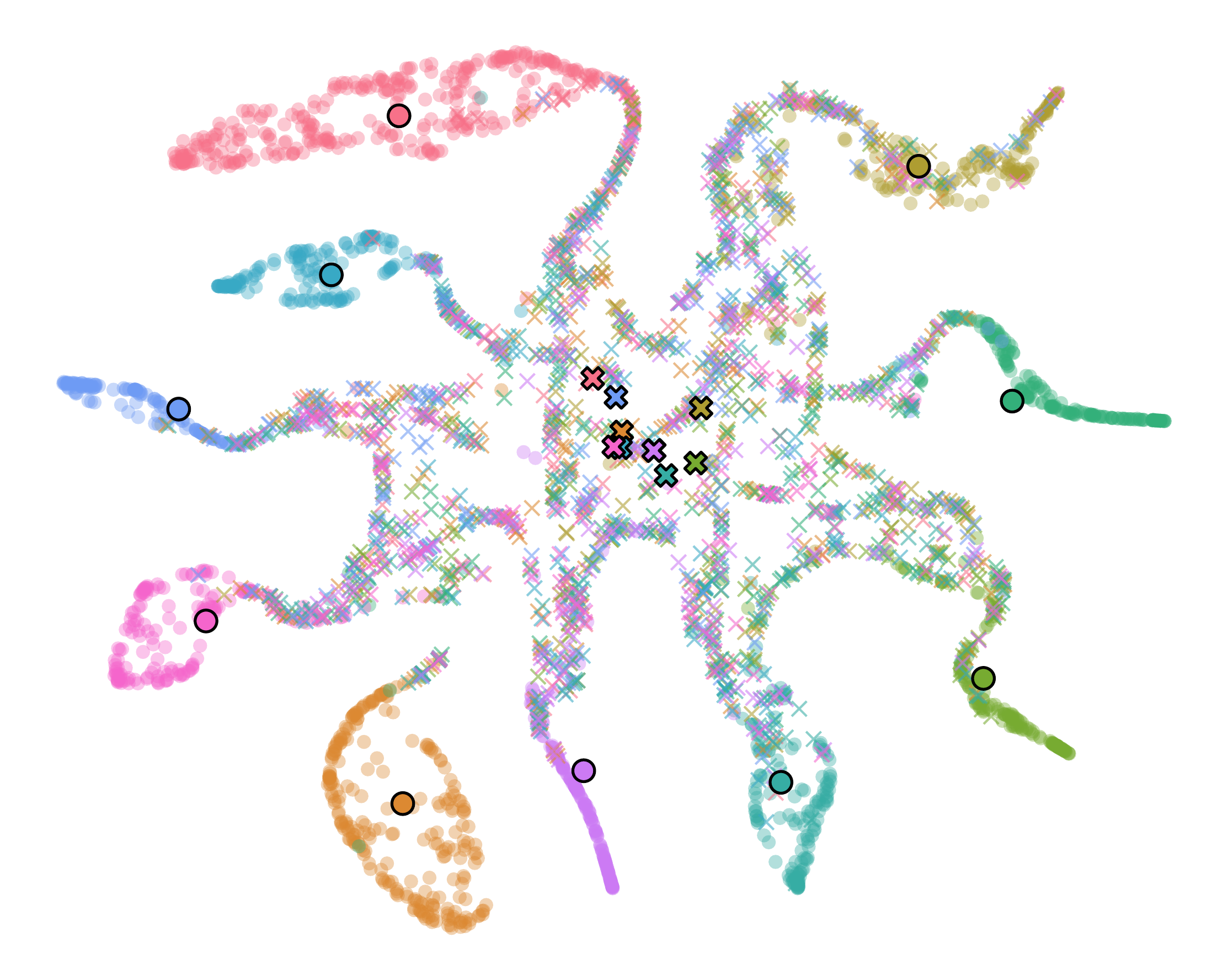}\\
\includegraphics[width=0.48\linewidth]{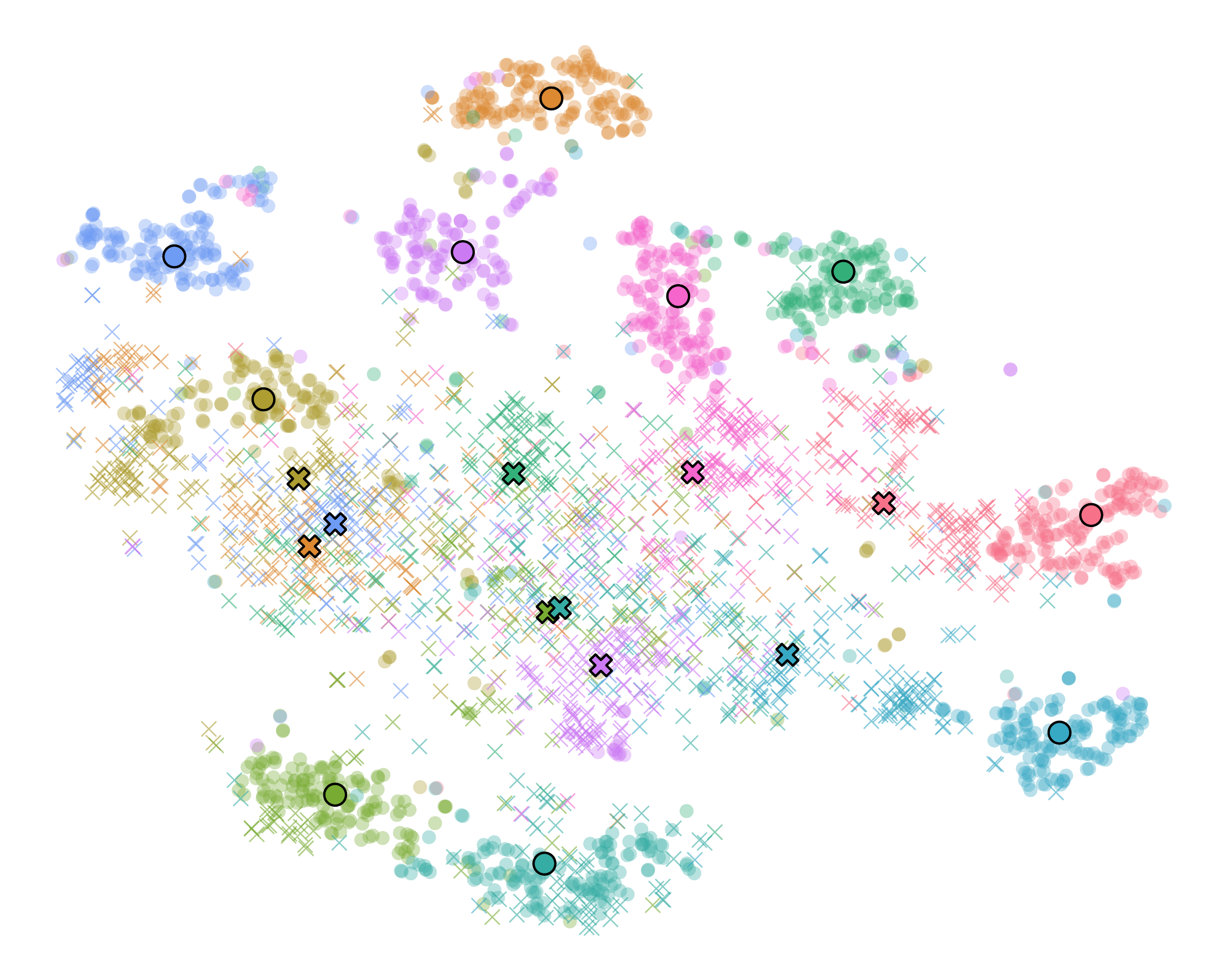}
\includegraphics[width=0.48\linewidth]{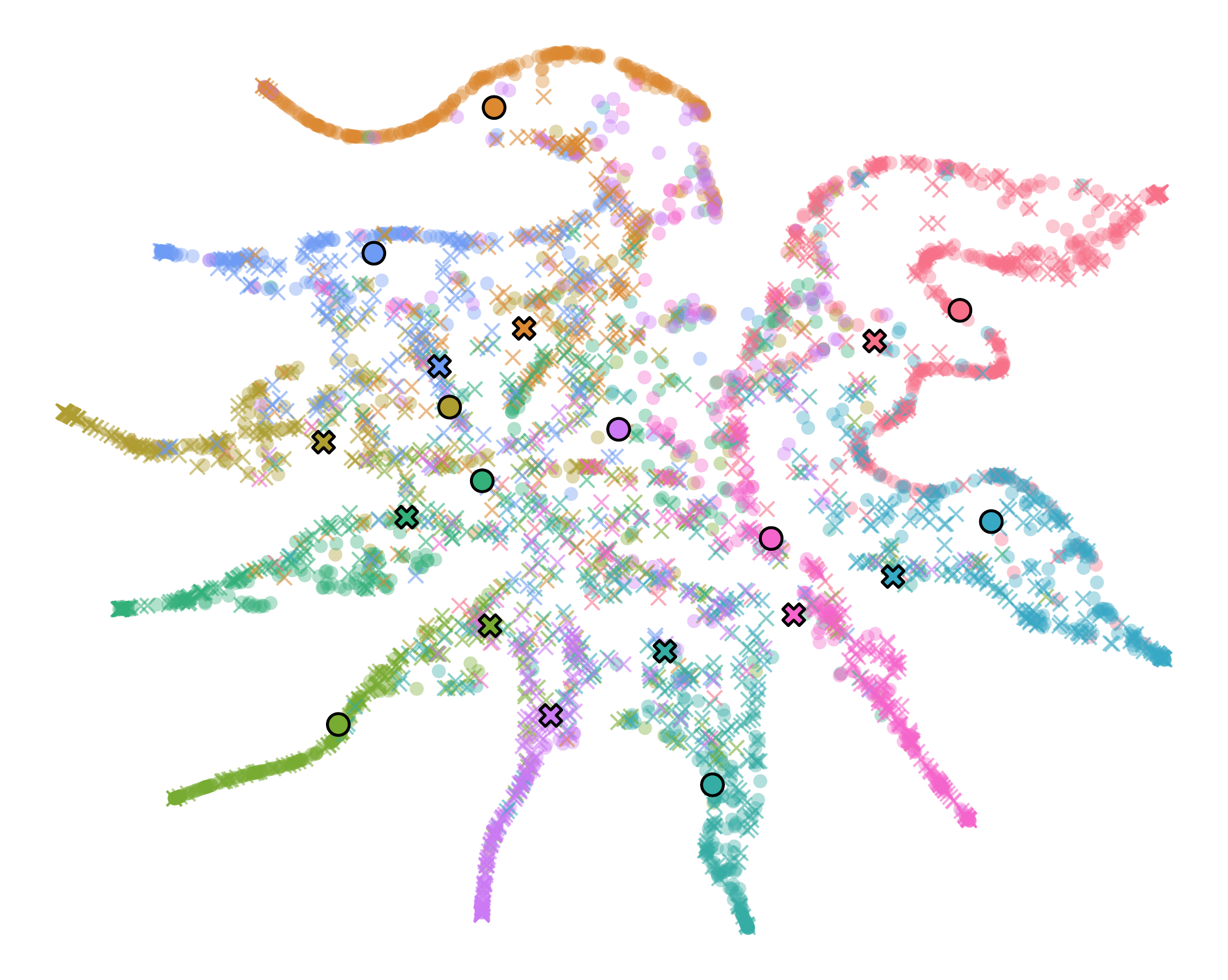}\\
\includegraphics[width=0.48\linewidth]{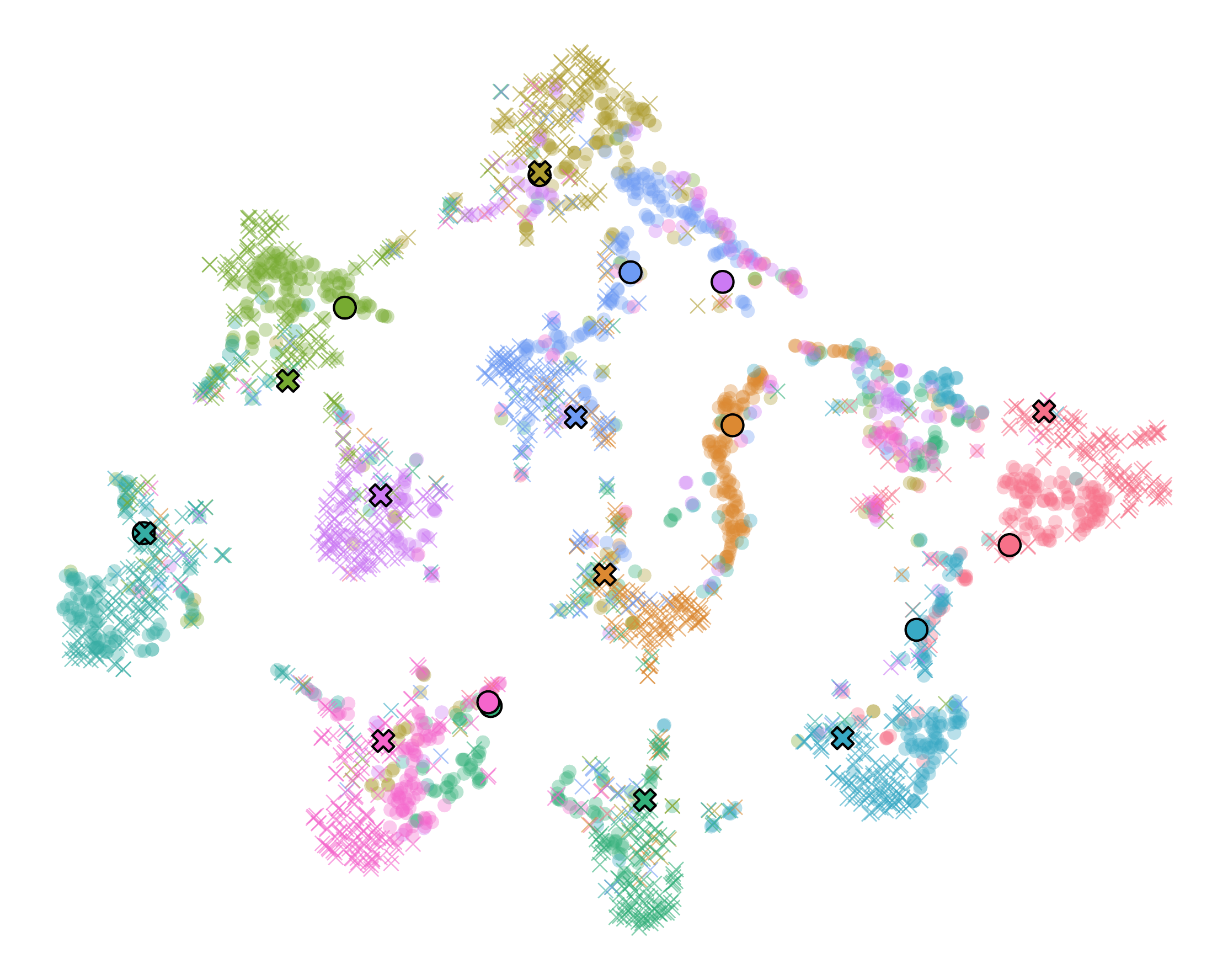}
\includegraphics[width=0.48\linewidth]{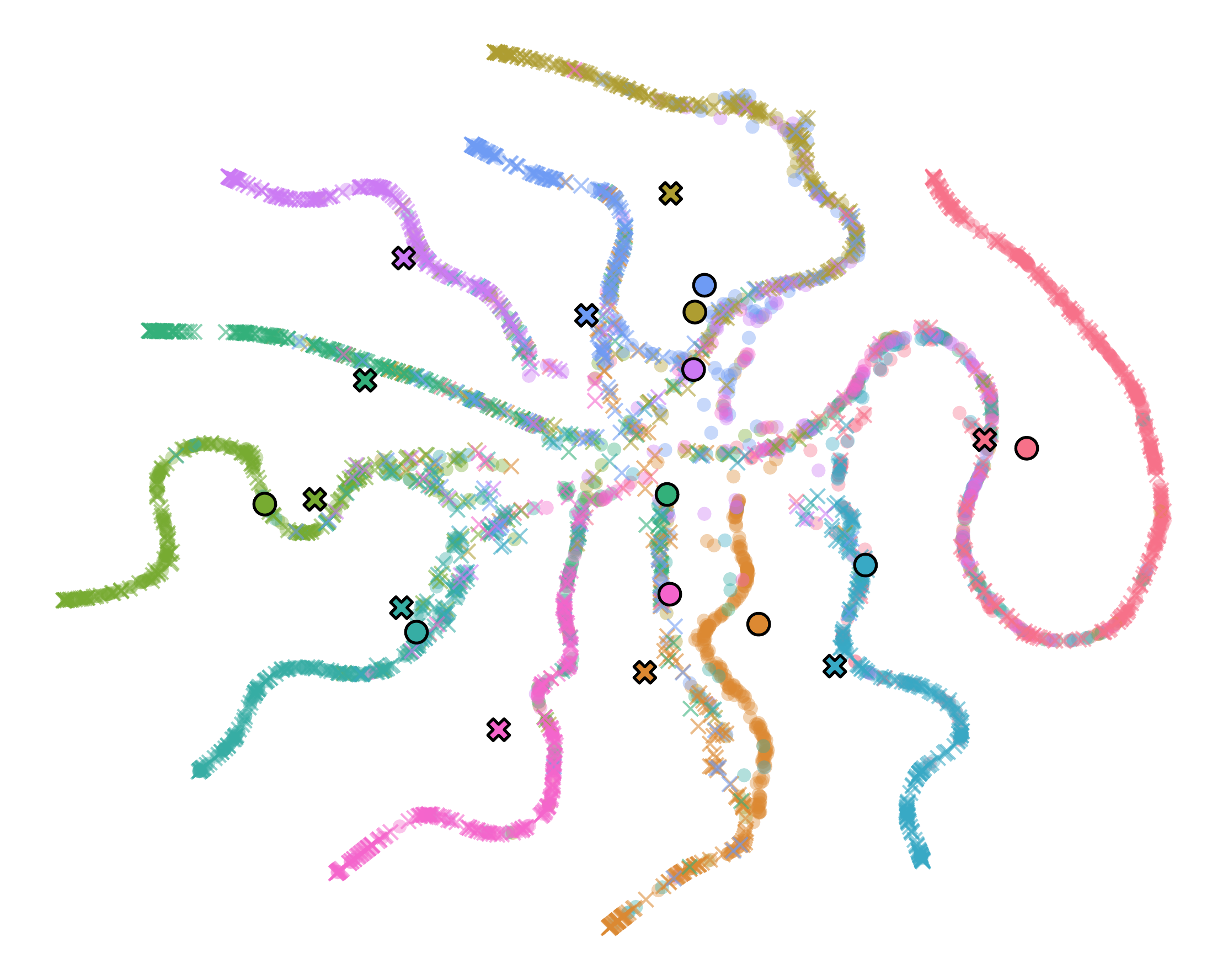}

\caption{t-SNE visualization of domain alignment between USPS (source, dots) and Syn (target, crosses). Columns show the feature space (left) and prediction space (right). Rows from top to bottom correspond to no adaptation, standard fine-tuning, and our JFPD method. 
JFPD produces tighter cross-domain clusters in the feature space and more consistent grouping in the prediction space, indicating better feature-prediction alignment and more reliable adaptation.
}
\label{fig:motivation}
\end{figure}

\section{Introduction}

\IEEEPARstart{D}{omain} adaptation (DA) \cite{wang2018deep, pan2010domain, li2024comprehensive} aims to transfer knowledge from a labeled source domain to an unlabeled or sparsely labeled target domain whose data distribution differs from that of the source. This capability is essential in many real-world applications where collecting labeled data for every deployment environment is impractical, such as medical imaging \cite{guan2021domain, perone2019unsupervised}, autonomous driving \cite{shan2019pixel}, and video understanding \cite{wang2024flow, wang2024high, dingjourney, ding2024quo, chen2024motion, zhu2024advancing, wang2024taylor, ding2025learnable}. In these scenarios, models trained on source data often suffer significant performance degradation when applied to target domains due to distribution shifts in appearance, feature representations, or semantic structures \cite{ding2024language, yu2023semi}. Addressing this challenge requires estimating and reducing the discrepancy between source and target domains.

Most existing DA approaches attempt to bridge this gap by aligning either feature representations or prediction distributions across domains \cite{farahani2021brief}. Feature-based methods \cite{chen2019joint, du2019ssf, li2024semantic} reduce domain shift by minimizing distances between intermediate representations, using metrics such as Euclidean distance, cosine similarity, or Maximum Mean Discrepancy (MMD) \cite{wang2021rethinking}. These methods aim to learn domain-invariant representations but often overlook whether aligned features correspond to consistent semantic predictions. Conversely, prediction-based methods align the model's outputs across domains by minimizing divergences such as Kullback-Leibler (KL) \cite{van2014renyi} or Jensen-Shannon (JS) divergence \cite{fuglede2004jensen}. While these approaches operate directly in the semantic space, they ignore the representation-level structures that give rise to those predictions. As a result, both strategies provide only partial views of domain discrepancy.

Beyond separating representation and prediction spaces, a more fundamental limitation persists in many existing methods: they implicitly assume that all samples provide equally reliable signals for alignment. Under domain shift, however, this assumption rarely holds. Target samples may exhibit uncertain predictions, ambiguous semantics, or poorly aligned representations due to noise, domain-specific variations, or model bias. Treating such samples uniformly during adaptation can introduce misleading correspondence signals and hinder effective transfer. Consequently, successful domain adaptation requires not only measuring domain discrepancy but also assessing the reliability of the signals used to estimate it.

Recent work has explored incorporating prediction information into domain alignment. 
For example, Conditional Adversarial Domain Adaptation (CDAN) \cite{cdanpe_nips} 
conditions adversarial feature alignment on classifier predictions and uses entropy-based 
weighting to reduce the influence of uncertain target samples. 
These approaches demonstrate that coupling feature representations with prediction information can improve domain alignment by conditioning feature adaptation on classifier outputs.
However, existing methods typically incorporate predictions as conditional signals for feature alignment rather than explicitly modeling joint discrepancies between feature and prediction spaces themselves \cite{ding2025graph}.
Moreover, the reliability of these alignment signals under domain shift 
is rarely quantified, even though both feature representations and predictions may become 
unreliable when the domain gap is large.

To address these open problems, we propose a trust-aware joint discrepancy formulation that explicitly measures both feature- and prediction-level divergences while estimating their reliability through complementary trust signals derived from prediction uncertainty and semantic feature alignment.
Unlike adversarial alignment methods, our formulation directly models domain discrepancy as a reliability-aware divergence between feature and prediction spaces.
We propose the \emph{Joint Feature-Prediction Discrepancy (JFPD)}, a unified formulation that captures domain shift by jointly considering representation divergence and prediction divergence while weighting their contributions according to sample-specific trust. Two complementary trust mechanisms are introduced to quantify the reliability of adaptation signals. First, \emph{uncertainty-aware trust} is derived from prediction entropy to down-weight unreliable or ambiguous predictions. Second, \emph{semantic-alignment trust} measures the consistency of feature representations with class prototypes, emphasizing samples that exhibit strong semantic alignment across domains. By integrating these signals, JFPD prioritizes confident and semantically consistent samples while suppressing noisy or uncertain correspondences, providing a reliability-aware estimate of domain discrepancy.

The motivation behind our work is illustrated in Fig.~\ref{fig:motivation}. 
Without adaptation, source and target samples form clearly separated clusters in both spaces, indicating a substantial domain gap. Standard fine-tuning partially reduces this gap but still produces scattered clusters and inconsistent predictions for many target samples. In contrast, JFPD 
yields significantly tighter cross-domain clusters in feature space and more coherent grouping in prediction space. This superiority arises because trust-aware discrepancy modeling focuses adaptation on samples that are both confident and semantically aligned, while reducing the influence of unreliable correspondences.

Importantly, the proposed discrepancy formulation serves not only as a diagnostic measure of domain shift but also as an optimization principle for adaptation. By incorporating JFPD into the training objective, the model is guided toward trustworthy regions of the target domain, leading to more stable and reliable cross-domain alignment. This unified perspective bridges discrepancy estimation and model optimization within a single reliability-aware framework.
Our contributions are: 

\renewcommand{\labelenumi}{\roman{enumi}.}
\begin{enumerate}[leftmargin=0.5cm]

\item A \textit{trust-aware perspective on domain discrepancy estimation} is introduced, which explicitly models the reliability of alignment signals under domain shift by considering the interaction between representation consistency and prediction confidence.

\item We propose the \textit{Joint Feature-Prediction Discrepancy (JFPD)}, a unified divergence that jointly captures feature-level and prediction-level discrepancies while incorporating complementary trust signals derived from prediction uncertainty and semantic alignment.

\item A \textit{trust-aware adaptation objective} is developed based on the proposed discrepancy, enabling the model to prioritize reliable cross-domain correspondences while suppressing noisy or ambiguous samples during adaptation.

\item Extensive experiments on domain adaptation benchmarks demonstrate that the proposed framework consistently achieves better adaptation performance and provides interpretable insights into cross-domain alignment dynamics.

\end{enumerate}









\section{Related Work}

Existing approaches fall into three categories: feature-level alignment, prediction-level adaptation, and methods that couple feature and prediction signals. Despite substantial progress, they typically rely on alignment signals whose reliability under domain shift is not modeled.

\textbf{Feature-level domain alignment.} A large body of work focuses on learning domain-invariant representations by aligning feature distributions across domains \cite{dan2023trust}. Early methods such as Deep Adaptation Networks (DAN) \cite{long2015learning} minimize MMD between source and target features. Subsequent extensions, including Joint Adaptation Networks (JAN) \cite{long2017deep}, align joint distributions of multiple network layers to capture more complex domain shifts. Adversarial approaches, such as Domain-Adversarial Neural Networks (DANN) and related variants \cite{ganin2016domain}, further encourage feature invariance by training a domain discriminator. Later works introduce more sophisticated adversarial strategies, such as CDAN \cite{cdanpe_nips} and gradually vanishing bridge methods \cite{cui2020gradually}, which incorporate classifier predictions into adversarial alignment.
Despite their effectiveness, feature-level alignment methods primarily focus on matching representation distributions and often assume that aligned features correspond to consistent semantic predictions. Under significant domain shifts, however, feature alignment alone may not guarantee a reliable semantic correspondence, particularly when predictions are uncertain or ambiguous.

\textbf{Prediction-level and confidence-aware adaptation.} Another line of work emphasizes adaptation in the prediction space by exploiting model confidence or uncertainty. Entropy minimization \cite{grandvalet2004semi,wang2019transferable, na2021fixbi, li2025seen} encourages confident predictions on target data and has been widely adopted in domain adaptation and test-time adaptation settings. For example, SHOT \cite{liang2020we} adapts a source-pretrained model by combining entropy minimization with self-supervised objectives, allowing adaptation without source data. TENT \cite{wang2021tent} further demonstrates the effectiveness of entropy-based adaptation during test time by updating normalization statistics to minimize prediction entropy. Other approaches incorporate sample-wise confidence measures to guide adaptation. For instance, CoWA-JMDS \cite{lee2022confidence} proposes a confidence score that combines prediction confidence with structural information from the feature distribution to weight samples during adaptation.
These approaches demonstrate the importance of prediction uncertainty in guiding adaptation. However, they primarily focus on the reliability of predictions and typically treat the underlying feature discrepancy as fixed or implicitly reliable.

\textbf{Joint feature-prediction modeling.} Several works attempt to combine feature and prediction signals to better capture domain shifts. Methods such as JAN \cite{long2017deep} align joint distributions across multiple layers, while CDAN \cite{cdanpe_nips} conditions adversarial alignment on classifier predictions, effectively coupling representation and prediction spaces. Prototype-based approaches also integrate semantic information by aligning features with class prototypes or cluster structures \cite{chen2019joint}. These methods demonstrate that jointly considering features and predictions can improve domain alignment.
Nevertheless, most existing methods incorporate prediction information primarily as conditional context for feature alignment rather than explicitly modeling how the reliability of features and predictions affects the estimation of domain discrepancy. In particular, when predictions are uncertain or representations are poorly aligned, the resulting alignment signals may become unreliable, potentially leading to misleading adaptation updates.

\textbf{Our trust-aware discrepancy modeling.} 
Instead of focusing solely on feature alignment, prediction confidence, or conditional coupling, we introduce a \emph{trust-aware discrepancy formulation} that explicitly models the reliability of both feature and prediction signals when estimating domain shift. The JFPD jointly measures representation divergence and prediction divergence while weighting them using complementary trust signals derived from prediction uncertainty and semantic alignment. This formulation allows adaptation to prioritize samples that provide reliable cross-domain correspondence while suppressing noisy or ambiguous ones.
By integrating reliability-aware discrepancy estimation with adaptation optimization, JFPD 
provides a unified perspective that connects representation alignment, prediction confidence, and sample-level trust within a single principled formulation.

\section{Method}

\textbf{Problem formulation.} Let $\gX$ denote the input space and $\gY = \{1,\dots,C\}$ the label space with $C$ classes.
We consider a standard domain adaptation setting consisting of a labeled source domain and an unlabeled or sparsely labeled target domain.
The source dataset is defined as
$\gD_s = \{(\vx_i^s, y_i^s)\}_{i=1}^{N_s}$,
where $\vx_i^s \in \gX$ and $y_i^s \in \gY$ are drawn from the joint distribution
$P_s(\vx,y)$.
The target dataset is $\gD_t = \{\vx_j^t\}_{j=1}^{N_t}$, where $\vx_j^t \sim P_t(\vx, y)$, with $\sim$ denoting that samples are drawn from the target joint distribution $P_t(\vx, y)$; the corresponding labels are unavailable or only partially observed.
In real-world applications such as cross-dataset visual recognition tasks, the source and target distributions typically differ, \ie,
$P_s(\vx,y) \neq P_t(\vx,y)$,
which leads to degraded performance when a model trained on $\gD_s$ is directly applied to $\gD_t$.
%
We consider a model of the form $h(\vx) = g(f(\vx))$, where $f:\gX\rightarrow\mathcal{F}$ is a feature extractor mapping inputs to a latent space $\mathcal{F}$, and $g:\mathcal{F}\rightarrow\Delta^{C-1}$ is a classifier that outputs a probability vector over $C$ classes, with $\Delta^{C-1} = \{\vp \in \mathbb{R}^C \mid p_c \ge 0, \sum_{c=1}^C p_c = 1\}$ denoting the $(C-1)$-dimensional probability simplex.
The overall goal of domain adaptation is to use labeled source data together with unlabeled (or sparsely labeled) target samples to learn a hypothesis $h$ that generalizes well on the target distribution.

A central challenge in domain adaptation is to quantify and reduce the discrepancy between the source and target domains.
Existing approaches typically measure domain shift either in the feature space, by aligning representation distributions,
or in the prediction space, by encouraging consistency between model outputs across domains.
While both strategies have shown effectiveness, they exhibit two limitations.
First, feature-level and prediction-level discrepancies are often modeled independently, even though the learned representation and classifier outputs are intrinsically coupled through the model $h(\vx)=g(f(\vx))$.
%
Second, many adaptation objectives treat target samples uniformly when computing discrepancy, despite the fact that different target instances exhibit varying degrees of reliability due to uncertainty, domain-specific noise, or semantic mismatch.
These observations motivate the need for a discrepancy measure that simultaneously accounts for representation shift and prediction inconsistency while selectively emphasizing reliable target samples.
To this end, we propose a \emph{trust-aware joint feature-prediction divergence (JFPD)} that integrates feature-level and prediction-level discrepancies and modulates their contributions using sample-specific trust estimates.
This method enables the model to focus adaptation on semantically consistent and confident regions of the target distribution while reducing the influence of unreliable samples.

Our method \textit{follows the standard unsupervised domain adaptation (UDA)} setting. 
During adaptation, pseudo-labels obtained from model predictions are used to associate target samples with source prototypes. 
%
Below, we formalize this idea through a unified discrepancy metric that jointly captures feature and prediction misalignment under trust-aware weighting.

\subsection{Trust-Aware Joint Feature-Prediction Divergence}
\label{sec:jfpd}

To characterize cross-domain discrepancy more comprehensively, we introduce a joint divergence measure that simultaneously accounts for mismatches in representation space and prediction space. The proposed Trust-Aware Joint Feature-Prediction Divergence integrates these two complementary views while allowing their contributions to be modulated by sample-specific trust estimates.

Consider a source sample $\vx^s$ and a target sample $\vx^t$. Let
$\vf_s = f(\vx^s)$, $\vf_t = f(\vx^t)$
denote their feature representations in the latent space $\mathcal{F}$, and let
$\vp_s = g(\vf_s)$, $\vp_t = g(\vf_t)$
be the corresponding softmax prediction vectors in the probability simplex $\Delta^{C-1}$.

\textbf{Feature-level divergence.} Feature-level discrepancy measures the structural mismatch between source and target representations in the latent space. Let $d(\cdot,\cdot)$ denote a distance metric in $\mathcal{F}$ (\eg, cosine or Euclidean distance). We define the normalized feature divergence as
\begin{equation}
d_{\text{feat}}(\vf_s,\vf_t)
=
\frac{d(\vf_s,\vf_t)}
{1+d(\vf_s,\vf_t)} .
\label{eq:feature_div}
\end{equation}

The transformation $\frac{x}{1+x}$ maps non-negative distances to the bounded interval $[0,1)$, preventing extreme values from dominating the discrepancy estimate and improving numerical stability during optimization.

\textbf{Prediction-level divergence.} In addition to representation mismatch, domain shift may also manifest as inconsistent model predictions. To quantify this effect, we compute divergence between the prediction distributions of the two samples. Let $D(\cdot,\cdot)$ denote a divergence between probability vectors (\eg, JS divergence). The normalized prediction divergence is defined as
\begin{equation}
d_{\text{pred}}(\vp_s,\vp_t)
=
\frac{D(\vp_s,\vp_t)}
{1+D(\vp_s,\vp_t)} .
\label{eq:pred_div}
\end{equation}

This term reflects semantic inconsistency between the model's outputs across domains, capturing shifts in decision boundaries or classifier calibration.

\textbf{Joint discrepancy formulation.}
Feature divergence and prediction divergence capture complementary aspects of domain mismatch. 
Because predictions are produced from features through the classifier, discrepancies in the representation space and prediction space are intrinsically coupled. Modeling them independently may therefore overlook important cross-space interactions during adaptation.

To address this problem, we integrate both components into a unified joint divergence measure. Let $\psi$ and $\phi$ denote trust weights associated with prediction reliability and feature alignment, respectively (defined in Section~\ref{sec:trust}). The proposed trust-aware JFPD is defined as
\begin{equation}
d_{\text{JFPD}}
=
\alpha  \psi  d_{\text{feat}}
+
(1-\alpha)  \phi  d_{\text{pred}},
\label{eq:jfpd}
\end{equation}
where $\alpha \in [0,1]$ balances the relative importance of feature-level and prediction-level discrepancies.

This formulation provides two desirable properties. First, it captures domain mismatch jointly in representation and decision spaces, yielding a more expressive measure of cross-domain discrepancy. Second, the trust weights $\psi$ and $\phi$ act as soft reliability indicators, reducing the influence of uncertain or poorly aligned samples during discrepancy estimation.

Consequently, JFPD provides a principled mechanism for measuring domain shift while selectively emphasizing reliable regions of the target distribution. In the following section, we describe how the trust weights are constructed to capture predictive confidence and semantic alignment.

\subsection{Trust Modeling}
\label{sec:trust}

Uniformly weighting samples in domain discrepancy estimation can introduce noisy gradients and unstable adaptation. 
We address this using \emph{trust-aware weighting} to modulate each sample pair’s contribution.
The trust modeling strategy is designed to capture two complementary aspects of reliability: predictive confidence and semantic alignment. These trust signals are then used to regulate the feature-level and prediction-level discrepancy terms introduced in Section~\ref{sec:jfpd}.

\textbf{Uncertainty-aware trust.} The first trust component measures the reliability of model predictions. Intuitively, if the classifier produces highly uncertain outputs for either the source or target sample, the corresponding pair should contribute less to the discrepancy estimation.
We quantify predictive uncertainty using entropy. For a prediction vector $\vp$, the entropy is defined as
\begin{equation}
H(\vp) = -\sum_{c=1}^{C} p_c \log p_c .
\end{equation}

Given the predictions $\vp_s$ and $\vp_t$ of a source-target pair, we define the uncertainty-aware trust weight as
\begin{equation}
\psi
=
\frac{1}{1 + H(\vp_s) + H(\vp_t)} .
\label{eq:trust_entropy}
\end{equation}
This formulation assigns higher trust to pairs whose predictions are confident in both domains, while downweighting ambiguous or low-confidence predictions. 

\textbf{Semantic-alignment trust.} While predictive confidence reflects classifier reliability, it does not directly measure whether the source and target samples are semantically aligned in the representation space. To capture this complementary signal, we introduce a feature-based trust weight that reflects semantic proximity between the two samples.
Based on the feature divergence in Eq.~(\ref{eq:feature_div}), the semantic-alignment trust is
\begin{equation}
\phi
=
\frac{1}{1 + d_{\text{feat}}(\vf_s,\vf_t)} .
\label{eq:trust_feat}
\end{equation}

This mechanism assigns higher trust to source-target pairs that are already close in the representation space, encouraging the model to focus adaptation on semantically consistent regions while suppressing noisy or poorly aligned pairs.

\textbf{Cross-guided trust design.} A key novelty of our framework is the cross-guided trust mechanism, where reliability signals from one space regulate adaptation in the other. Specifically, \textit{predictive confidence} ($\psi$) modulates the feature-level divergence, and \textit{semantic alignment} ($\phi$) modulates the prediction-level divergence.
This cross-guided strategy provides two advantages. First, prediction confidence helps prevent unreliable classifier outputs from driving representation alignment. Second, semantic proximity in the feature space helps stabilize prediction-level discrepancy by emphasizing samples that are likely to share the same underlying semantics.

Through this complementary design, the trust weights act as soft reliability indicators that selectively emphasize trustworthy source-target pairs while reducing the influence of uncertain or mismatched samples. When incorporated into the joint divergence formulation of Eq.~(\ref{eq:jfpd}), this mechanism enables more stable and reliable domain alignment.

\subsection{Trust-Aware Adaptation Objective}
\label{sec:objective}

We now translate this discrepancy into a training objective that enables effective domain adaptation.

\textbf{Source prototypes.} Directly computing pairwise divergence between all source and target samples can be computationally expensive and may introduce noise due to irrelevant sample correspondences. 
Instead of aligning every source-target pair, which is computationally expensive and noisy, we approximate the source distribution using class prototypes.
For each class $c \in \{1,\dots,C\}$, we compute a source feature prototype
\begin{equation}
\vz_c^s = \frac{1}{|\gD_s^c|} 
\sum_{\vx_i^s \in \gD_s^c} f(\vx_i^s),
\label{eq:zp}
\end{equation}
where $\gD_s^c$ denotes the subset of source samples belonging to class $c$.
Similarly, we define a prediction prototype
\begin{equation}
\vp_c^s =
\frac{1}{|\gD_s^c|}
\sum_{\vx_i^s \in \gD_s^c} g(f(\vx_i^s)).
\label{eq:pp}
\end{equation}

These prototypes provide compact representations of class semantics in feature and prediction spaces, enabling target samples to align with their corresponding source categories.

\textbf{Trust-aware discrepancy minimization.} Let $\vx^t$ be a target sample with feature representation
$\vf_t = f(\vx^t)$
and prediction
$\vp_t = g(\vf_t)$.
A pseudo-label is then assigned according to
\begin{equation}
\hat{y}_t = \arg\max_c p_t^{(c)},
\label{eq:pseudo-label}
\end{equation}
and the target sample is aligned with its corresponding source prototypes $(\vz_{\hat{y}^t}^s, \vp_{\hat{y}^t}^s)$.
Following the joint divergence formulation, we define the per-sample adaptation loss as
\begin{equation}
\mathcal{L}_{\text{JFPD}}
=
\alpha  \psi  d_{\text{feat}}(\vf_t,\vz_{\hat{y}^t}^s)
+
(1-\alpha)  \phi  d_{\text{pred}}(\vp_t,\vp_{\hat{y}^t}^s),
\label{eq:jfpd_loss}
\end{equation}
where $\psi$ and $\phi$ are the trust weights defined in Section~\ref{sec:trust}. This objective encourages the model to reduce both representation mismatch and prediction inconsistency between target samples and their corresponding source class prototypes.

\textbf{Cross-guided alignment.} 
The loss in Eq.~(\ref{eq:jfpd_loss}) implements the cross-guided trust mechanism described earlier, where feature divergence is modulated by prediction confidence,
\begin{equation}
\mathcal{L}_{\text{PGFD}}
=
\psi  d_{\text{feat}}(\vf_t,\vz_{\hat{y}^t}^s),
\end{equation}
which we refer to as \emph{prediction-guided feature discrepancy (PGFD)}.
This term encourages representation alignment primarily when the classifier produces reliable predictions.

Conversely, the prediction-level divergence is weighted by feature alignment,
\begin{equation}
\mathcal{L}_{\text{FGPD}}
=
\phi  d_{\text{pred}}(\vp_t,\vp_{\hat{y}^t}^s),
\end{equation}
which we denote as \emph{feature-guided prediction divergence (FGPD)}.
This term emphasizes prediction consistency for samples that are semantically close in the feature space.

By integrating these complementary components, the overall objective selectively focuses adaptation on reliable and semantically consistent regions of the target distribution.

\textbf{Training objective.} The adaptation objective is obtained by averaging the per-sample losses over the target dataset:
\begin{equation}
\mathcal{L}
=
\frac{1}{N_t}
\sum_{j=1}^{N_t}
\mathcal{L}_{\text{JFPD}}(\vx_j^t).
\label{eq:final-loss}
\end{equation}
Below, we provide a theoretical interpretation of JFPD. 

\subsection{Theoretical Justification}
\label{sec:theory}

Our analysis shows that JFPD acts as a \emph{trust-weighted surrogate} for the domain divergence term in classical domain adaptation theory. 
Minimizing the expected JFPD therefore reduces the upper bound on the target-domain risk while emphasizing reliable regions of the target distribution.

\textbf{Generalization bound.}
Let $\epsilon_S(h)$ and $\epsilon_T(h)$ denote the expected source and target risks of a hypothesis $h \in \mathcal{H}$. 
A fundamental result in domain adaptation theory \cite{ben2010theory} states that
\begin{equation}
\epsilon_T(h)
\le
\epsilon_S(h)
+
\frac{1}{2} d_{\mathcal{H}\Delta\mathcal{H}}(P_s,P_t)
+
\lambda ,
\label{eq:ben_david_bound}
\end{equation}
where $d_{\mathcal{H}\Delta\mathcal{H}}(P_s,P_t)$ denotes the $\mathcal{H}\Delta\mathcal{H}$ divergence between the source and target distributions, and 
$\lambda = \min_{h \in \mathcal{H}} (\epsilon_S(h)+\epsilon_T(h))$ represents the joint error of the optimal hypothesis.
This bound shows that the target error can be controlled by minimizing both the source risk and the distribution discrepancy.
In practical domain adaptation settings, the source model is pretrained; thus, adaptation primarily focuses on reducing the divergence between source and target domains.

\textbf{JFPD as a surrogate divergence.}
Direct computation of $d_{\mathcal{H}\Delta\mathcal{H}}$ is generally intractable because it requires evaluating classifier disagreement over the entire hypothesis space $\mathcal{H}$. 
Practical methods therefore rely on tractable empirical surrogates that approximate cross-domain mismatch.

In our framework, the source distribution is summarized using class-level prototypes in both feature and prediction spaces (Eqs.~(\ref{eq:zp}) and~(\ref{eq:pp})). 
For each target sample $\vx^t$, the corresponding source prototypes $(\vz_{\hat{y}^t}^s,\vp_{\hat{y}^t}^s)$ provide reference representations that characterize the class structure of the source domain.
Feature-level divergence measures structural mismatch in the representation space, while prediction-level divergence captures semantic disagreement between probability distributions. 
By combining these complementary signals, we define the per-sample joint discrepancy
\begin{equation}
d_{\text{JFPD}}(\vx^t)
\!=\!
\alpha\psi(\vx^t) d_{\text{feat}}(\vf_t,\vz_{\hat{y}^t}^s)
\!+\!
(1\!-\!\alpha)\phi(\vx^t) d_{\text{pred}}(\vp_t,\vp_{\hat{y}^t}^s), \nonumber
\label{eq:jfpd_prototype}
\end{equation}
where $\alpha\in[0,1]$ balances feature and prediction discrepancies.
The trust weights $\psi(\vx^t)$ and $\phi(\vx^t)$ downweight unreliable samples and emphasize semantically consistent target regions. 

\textbf{Trust-weighted divergence surrogate.}  
We show that the expected JFPD acts as a trust-weighted surrogate for the $\mathcal{H}\Delta\mathcal{H}$ divergence.  
This establishes a connection between minimizing JFPD and reducing the domain discrepancy in Eq.~(\ref{eq:ben_david_bound}).

\begin{proposition}[Trust-weighted surrogate bound]
\label{prop:trust_bound_full}
Let $d_{\text{feat}}(\cdot,\cdot)$ and $d_{\text{pred}}(\cdot,\cdot)$ be bounded divergences satisfying 
$0 \!\le\! d_{\text{feat}}, d_{\text{pred}} \!\le\! 1$.  
Let $\psi(\vx^t), \phi(\vx^t) \!\in\! (0,1]$ denote trust weights.  
Assume there exist constants $\gamma_1, \gamma_2 \!>\! 0$ such that
$\gamma_1 \!\le\! \psi(\vx^t) \!\le\! 1$ and $\gamma_2 \!\le\! \phi(\vx^t) \!\le\! 1$ for all $\vx^t$.  
Then there exists a residual constant $\xi \!\ge\! 0$ such that
\begin{align}
d_{\mathcal{H}\Delta\mathcal{H}}(P_s,P_t)
& \le 
\frac{1}{\gamma_1} \mathbb{E}_{\vx^t \sim P_t}
\big[\psi(\vx^t) d_{\text{feat}}(\vf_t, \vz_{\hat{y}^t}^s)\big] \nonumber \\
&\quad + \frac{1}{\gamma_2} \mathbb{E}_{\vx^t \sim P_t}
\big[\phi(\vx^t) d_{\text{pred}}(\vp_t, \vp_{\hat{y}^t}^s)\big]
+ \xi. \nonumber
\end{align}
\end{proposition}

\begin{proof}
Feature and prediction divergences capture structural and semantic differences between source and target representations.  
Under mild regularity conditions, they provide empirical surrogates for hypothesis disagreement.  
Hence, there exists a residual constant $\xi \ge 0$ such that
\[
d_{\mathcal{H}\Delta\mathcal{H}}(P_s,P_t)
\le
\mathbb{E}[d_{\text{feat}}(\vf_t, \vz_{\hat{y}^t}^s)]
+
\mathbb{E}[d_{\text{pred}}(\vp_t, \vp_{\hat{y}^t}^s)]
+
\xi.
\]

The residual $\xi$ accounts for approximation errors due to surrogate divergence estimation, imperfect prototypes, and outlier samples.
Since $\psi(\vx^t) \!\!\ge\!\! \gamma_1$, we have
$d_{\text{feat}}(\vf_t, \vz_{\hat{y}^t}^s)
\!\!\le\!\! \frac{1}{\gamma_1} \psi(\vx^t) d_{\text{feat}}(\vf_t, \vz_{\hat{y}^t}^s)$,
and taking expectations yields
$\mathbb{E}[d_{\text{feat}}] \!\le\! \frac{1}{\gamma_1} \mathbb{E}[\psi(\vx^t) d_{\text{feat}}]$.
A similar argument applies to $d_{\text{pred}}$, giving
$\mathbb{E}[d_{\text{pred}}] \!\le\! \frac{1}{\gamma_2} \mathbb{E}[\phi(\vx^t) d_{\text{pred}}]$.
Substituting these into the surrogate bound completes the proof.
\end{proof}

We define $\alpha = \frac{\gamma_2}{\gamma_1 + \gamma_2}$ (so that $1-\alpha = \frac{\gamma_1}{\gamma_1 + \gamma_2}$). 
With this choice, the divergence can be bounded as
$
d_{\mathcal{H}\Delta\mathcal{H}}(P_s,P_t)
\le
M \mathbb{E}[d_{\text{JFPD}}] + \xi,
$
where $M = \max\!\left\{\tfrac{1}{\gamma_1}, \tfrac{1}{\gamma_2}\right\}$ and
$\mathbb{E}[d_{\text{JFPD}}]=\alpha\mathbb{E}[\psi d_{\text{feat}}]
+
(1-\alpha) \mathbb{E}[\phi d_{\text{pred}}]
$.

\textbf{Effect of trust weighting.}
Trust weighting has an important stabilizing effect on divergence estimation. 

\begin{lemma}[Suppression of unreliable samples]
\label{lemma:trust_suppression}

Let $d_{\text{feat}}, d_{\text{pred}} \in [0,1]$ and $\psi,\phi \in (0,1]$.
Then
\[
\mathbb{E}[d_{\text{JFPD}}]=\mathbb{E}[\alpha \psi d_{\text{feat}} + (1-\alpha)\phi d_{\text{pred}}]
\le
\mathbb{E}[\alpha d_{\text{feat}} + (1-\alpha)d_{\text{pred}}].
\]

\end{lemma}

\begin{proof}
Since $0<\psi,\phi\le1$, we have $\psi d_{\text{feat}}\le d_{\text{feat}}$ and $\phi d_{\text{pred}}\le d_{\text{pred}}$. 
Taking expectations and using linearity completes the proof.
\end{proof}

This result formalizes the intuition that trust weighting suppresses unreliable or uncertain samples, preventing them from dominating the discrepancy estimate and improving the stability of domain alignment.

\textbf{Target risk reduction via JFPD minimization.}
Combining Proposition~\ref{prop:trust_bound_full} with the generalization bound in Eq.~(\ref{eq:ben_david_bound}) yields the following result.

\begin{corollary}[Target risk bound under trust-aware adaptation]
\label{cor:target_risk}

Assume that the ideal joint error $\lambda$ is small. Then, minimizing the expected JFPD provides an upper bound on the target-domain risk:
$\epsilon_T(h)
\le
\epsilon_S(h)
+
\frac{M}{2}\mathbb{E}[d_{\text{JFPD}}]
+
\frac{1}{2}\xi
+
\lambda$.

\end{corollary}

\begin{proof}
From the domain adaptation generalization bound (Eq.~(\ref{eq:ben_david_bound})), we have
$\epsilon_T(h)
\le
\epsilon_S(h)
+
\frac{1}{2} d_{\mathcal{H}\Delta\mathcal{H}}(P_s,P_t)
+
\lambda$.
By Proposition~\ref{prop:trust_bound_full},
$d_{\mathcal{H}\Delta\mathcal{H}}(P_s,P_t)
\le
M\mathbb{E}[d_{\text{JFPD}}]
+
\xi$.
Substituting this bound completes the proof.
\end{proof}

This result provides theoretical justification for the proposed adaptation objective. 
Minimizing JFPD simultaneously reduces representation mismatch in feature space and semantic disagreement in prediction space. 
The trust weights further guide optimization toward reliable and semantically consistent target samples, while suppressing uncertain or mismatched regions.
Consequently, the proposed trust-aware JFPD (Eq. (\ref{eq:jfpd})) offers a principled mechanism for controlling the target-domain generalization error while improving robustness to noisy or unreliable target instances.

\subsection{Algorithm and Implementation}
\label{sec:implementation}

This section describes the optimization procedure of the proposed trust-aware adaptation framework and provides practical implementation details. 
We 
adopt cosine distance for feature divergence, and use JS divergence for prediction divergence.
%
The overall training procedure consists of two stages: \emph{source pretraining} and \emph{trust-aware target adaptation}.

\textbf{Source pretraining.}
We first train the backbone model $h(\vx)=g(f(\vx))$ on the labeled source dataset $\gD_s$. 
The feature extractor $f(\cdot)$ and classifier $g(\cdot)$ are optimized using the standard cross-entropy loss
\begin{equation}
\mathcal{L}_{\text{CE}}
=
-\mathbb{E}_{(\vx_i^s,y_i^s)\sim\gD_s}
\log p(y_i^s|\vx_i^s).
\end{equation}

After pretraining, the learned model provides stable features and prediction distributions for the source domain. 
These outputs are used to construct class-level prototypes in both feature and prediction spaces as defined in Eqs.~(\ref{eq:zp}) and~(\ref{eq:pp}). 
The prototypes summarize the semantic structure of the source domain and serve as alignment anchors for target samples.

\textbf{Trust-aware target adaptation.} During adaptation, the pretrained model is fine-tuned using unlabeled target data by minimizing the proposed JFPD. 
For each target sample, we compute its feature representation and prediction vector. 
The corresponding source prototypes $(\vz^s_{\hat{y}_t},\vp^s_{\hat{y}_t})$ are retrieved (via Eqs. (\ref{eq:zp}), (\ref{eq:pp}) and (\ref{eq:pseudo-label})) to evaluate cross-domain discrepancy. 
Feature divergence and prediction divergence are computed using Eqs.~(\ref{eq:feature_div}) and~(\ref{eq:pred_div}). 
The trust weights $\psi$ and $\phi$ are estimated via Eqs.~(\ref{eq:trust_entropy}) and~(\ref{eq:trust_feat}), which quantify prediction confidence and semantic alignment, respectively.
The resulting per-sample adaptation loss is Eq. (\ref{eq:jfpd_loss}).
The final training objective (Eq. (\ref{eq:final-loss}) minimizes the average loss over target samples.

\textbf{Dynamic prototype estimation.}
To improve robustness and avoid overfitting to fixed summary statistics, we use a dynamic prototype estimation strategy. 
Instead of computing prototypes once from the entire source dataset, we construct \emph{mini-batch prototypes} by randomly sampling a subset of source samples from each class during each training iteration. 
These stochastic prototypes encourage diverse cross-domain alignment and help the model adapt to evolving target representations.

\textbf{Optimization.}
The model parameters (feature extractor and classifier) 
are updated using stochastic gradient descent with mini-batches of target samples. 
Trust weights are recomputed at every iteration based on the current predictions and features. 
This adaptive mechanism gradually shifts the optimization focus toward reliable and semantically consistent target samples, while suppressing uncertain or mismatched instances.

\begin{table*}[tbp]
\setlength{\tabcolsep}{0.2em}
\renewcommand{\arraystretch}{0.50}
\centering
\caption{Results on Digits.
We evaluate two backbones: a ViT-S-4/9-420 (a small Vision Transformer with $4\!\times\!4$ patch size, 9 layers, and 420-dimensional embeddings) and a lightweight VGG-style CNN (three Conv $\!\rightarrow\!$ ReLU $\!\rightarrow\!$ Pool blocks). 
}
\resizebox{0.85\textwidth}{!}{
\begin{tabular}{ll|cccc||cccc||cccc||cccc||cccc}
\toprule
& & \multicolumn{20}{c}{{Target}} \\
\addlinespace[0.5ex]
\cline{3-22}
\addlinespace[0.5ex]
\multirow{2}{*}{{Model}} & \multirow{2}{*}{{Source}} & \multicolumn{4}{c||}{No fine-tuning}
& \multicolumn{4}{c||}{Standard fine-tuning}
& \multicolumn{4}{c||}{\textbf{FGPD fine-tuning}}
& \multicolumn{4}{c||}{\textbf{PGFD fine-tuning}}
& \multicolumn{4}{c}{\textbf{Full JFPD fine-tuning (ours)}} \\
\addlinespace[0.5ex]
\cline{3-22}
\addlinespace[0.5ex]
& & MNIST & SVHN & USPS & Syn
& MNIST & SVHN & USPS & Syn
& MNIST & SVHN & USPS & Syn
& MNIST & SVHN & USPS & Syn
& MNIST & SVHN & USPS & Syn \\
\midrule

\multirow{5}{*}{ViT-S}
& MNIST 
& 99.65 & 19.88 & 94.02 & 18.80  
& -- & 94.06 & 97.46 & 75.35 
& -- & 95.01 & 98.16 & 88.50 
& -- & 76.64 & 97.76 & 60.40 
& -- & \textbf{95.82} & \textbf{98.06} & \textbf{89.30} \\

& SVHN 
& 70.34 & 93.79 & 76.53 & 39.35  
& 99.49 & -- & 97.01 & 80.15 
& 99.71 & -- & 97.91 & 87.95 
& 95.63 & -- & 95.62 & 52.60 
& \textbf{99.63} & -- & \textbf{97.86} & \textbf{87.30} \\

& USPS 
& 54.17 & 16.15 & 96.66 & 17.15 
& 99.32 & 91.58 & -- & 66.15 
& 99.43 & 92.30 & -- & 77.50 
& 81.70 & 39.40 & -- & 23.65 
& \textbf{99.43} & \textbf{93.38} & -- & \textbf{79.90} \\

& Syn 
& 20.33 & 30.95 & 62.33 & 73.30 
& 99.29 & 91.92 & 96.66 & -- 
& 99.47 & 93.80 & 97.36 & -- 
& 69.96 & 35.60 & 78.03 & -- 
& \textbf{99.50} & \textbf{93.02} & \textbf{97.56} & -- \\

\midrule

\multirow{5}{*}{CNN}
& MNIST
& 99.61 & 20.76 & 85.70 & 21.40
& -- & 91.57 & 98.16 & 74.50
& -- & 92.63 & 98.46 & 88.75
& -- & 70.21 & 97.11 & 62.55
& -- & \textbf{93.89} & \textbf{98.31} & \textbf{93.75} \\

& SVHN
& 68.65 & 94.51 & 75.44 & 43.35
& 99.37 & -- & 97.36 & 81.90
& 99.50 & -- & 97.91 & 91.25
& 93.67 & -- & 90.88 & 57.00
& \textbf{99.58} & -- & \textbf{98.16} & \textbf{94.25} \\

& USPS
& 94.31 & 20.62 & 97.66 & 19.60
& 99.46 & 91.44 & -- & 70.50
& 99.56 & 91.78 & -- & 83.65
& 98.28 & 47.10 & -- & 39.85
& \textbf{99.60} & \textbf{93.36} & -- & \textbf{91.45} \\

& Syn
& 28.82 & 38.59 & 63.08 & 94.85
& 99.28 & 91.13 & 96.76 & --
& 99.44 & 91.40 & 97.76 & --
& 85.99 & 63.14 & 84.16 & --
& \textbf{99.49} & \textbf{93.67} & \textbf{97.81} & -- \\

\bottomrule
\end{tabular}
}
\label{tab:vit_cnn_digits_all}
\end{table*}

\begin{table}[tbp]
\setlength{\tabcolsep}{0.1em}
\renewcommand{\arraystretch}{0.50}
\centering
\caption{Comparison with SOTA methods on Office-Home.}
\label{tab:officehome_full}
\resizebox{\linewidth}{!}{
\begin{tabular}{l|cccccccccccc|c}
\toprule
Method & Ar2Cl & Ar2Pr & Ar2Re & Cl2Ar & Cl2Pr & Cl2Re & Pr2Ar & Pr2Cl & Pr2Re & Re2Ar & Re2Cl & Re2Pr & Avg. \\
\midrule
ResNet-50 \cite{he2016deep}& 44.9 & 66.3 & 74.3 & 51.8 & 61.9 & 63.6 & 52.4 & 39.1 & 71.2 & 63.8 & 45.9 & 77.2 & 59.4 \\
MinEnt \cite{minent_pr}& 51.0 & 71.9 & 77.1 & 61.2 & 69.1 & 70.1 & 59.3 & 48.7 & 77.0 & 70.4 & 53.0 & 81.0 & 65.8 \\
SAFN \cite{safn_iccv19}& 52.0 & 71.7 & 76.3 & 64.2 & 69.9 & 71.9 & 63.7 & 51.4 & 77.1 & 70.9 & 57.1 & 81.5 & 67.3 \\
CDAN+E \cite{cdanpe_nips}& 54.6 & 74.1 & 78.1 & 63.0 & 72.2 & 74.1 & 61.6 & 52.3 & 79.1 & 72.3 & 57.3 & 82.8 & 68.5 \\
DCAN \cite{dcan_aaai}& 54.5 & 75.7 & 81.2 & 67.4 & 74.0 & 76.3 & 67.4 & 52.7 & 80.6 & 74.1 & 59.1 & 83.5 & 70.5 \\
BNM \cite{bnm_cvpr}& 56.7 & 77.5 & 81.0 & 67.3 & 76.3 & 77.1 & 65.3 & 55.1 & 82.0 & 73.6 & 57.0 & 84.3 & 71.1 \\
SHOT \cite{shot_cvpr}& 57.1 & 78.1 & 81.5 & 68.0 & 78.2 & 78.1 & 67.4 & 54.9 & 82.2 & 73.3 & 58.8 & 84.3 & 71.8 \\
ATDOC-NA \cite{atdocna_cvpr}& 58.3 & 78.8 & 82.3 & 69.4 & 78.2 & 78.2 & 67.1 & 56.0 & 82.7 & 72.0 & 58.2 & 85.5 & 72.2 \\
CoWA-JMDS \cite{lee2022confidence}& 56.9 & 78.4 & 81.0 & 69.1 & 80.0 & 79.9 & 67.7 & 57.2 & 82.4 & 72.8 & 60.5 & 84.5 & 72.5 \\
Fixbi \cite{na2021fixbi}& 58.1 & 77.3 & 80.4 & 67.7 & 79.5 & 78.1 & 65.8 & 57.9 & 81.7 & 76.4 & 62.9 & 86.7 & 72.7 \\
DUET \cite{duet_nips}& 73.6 & 90.4 & 91.0 & 83.6 & 90.7 & 90.9 & 82.7 & 73.7 & 91.2 & 83.6 & 74.0 & 91.2 & 84.7 \\
\midrule
ViT-B \cite{dosovitskiy2020image}& 54.7 & 83.0 & 87.2 & 77.3 & 83.4 & 85.6 & 74.4 & 50.9 & 87.2 & 79.6 & 54.8 & 88.8 & 75.5 \\
CDTrans-B \cite{cdtrans_iclr}& 68.8 & 85.0 & 86.9 & 81.5 & 87.1 & 87.3 & 79.6 & 63.3 & 88.2 & 82.0 & 66.0 & 90.6 & 80.5 \\
TVT-B \cite{tvt_cvpr}& 74.9 & 86.8 & 89.5 & 82.8 & 88.0 & 88.3 & 79.8 & 71.9 & 90.1 & 85.5 & 74.6 & 90.6 & 83.6 \\
SSRT-B \cite{ssrt_cvpr}& 75.2 & 89.0 & 91.1 & 85.1 & 88.3 & 90.0 & 85.0 & 74.2 & 91.3 & 85.7 & 78.6 & 91.8 & 85.4 \\
DCST \cite{dcst_nn} & 76.1 & 90.6 & 91.3 & 86.5 & 91.4 & 92.3 & 84.5 & 73.3 & 91.7 & 86.0 & 74.4 & 93.2 & 86.0\\
TSGF-B \cite{tsgf_pr}& 81.2 & 92.6 & 92.3 & 86.4 & 92.7 & 92.1 & 84.4 & 79.8 & 93.3 & 87.2 & 82.1 & 94.2 & 88.2 \\
FFTAT-B \cite{fftat_wacv}& 83.2 & \textbf{92.9} & 95.2 & 91.1 & 93.5 & 95.2 & 89.7 & 85.0 & 94.9 & 93.0 & 87.5 & 95.9 & 91.4 \\
\midrule
\textbf{JFPD (ours)} & \textbf{86.9} & \textbf{92.9} & \textbf{95.4} & \textbf{91.4} & \textbf{94.0} & \textbf{95.6} & \textbf{90.2} & \textbf{86.6} & \textbf{95.2} & \textbf{93.1} & \textbf{88.4} & \textbf{96.5} & \textbf{92.2} \\
\bottomrule
\end{tabular}
}
\end{table}

\begin{table}[tbp]
\setlength{\tabcolsep}{0.2em}
\renewcommand{\arraystretch}{0.50}
\centering
\caption{Comparison with SOTA methods on  VisDA-17.}
\label{tab:visda_comparison}
\resizebox{0.99\linewidth}{!}{
\begin{tabular}{l|cccccccccccc|c}
\toprule
Method & plane & bcycl & bus & car & horse & knife & mcycl & person & plant & sktbrd & train & truck & Avg. \\
\midrule
DANN \cite{ganin2016domain}& 81.9 & 77.7 & 82.8 & 44.3 & 81.2 & 29.5 & 65.1 & 28.6 & 51.9 & 54.6 & 82.8 & 7.8 & 57.4 \\
MinEnt \cite{minent_pr}& 80.3 & 75.5 & 75.8 & 48.3 & 77.9 & 27.3 & 69.7 & 40.2 & 46.5 & 46.6 & 79.3 & 16.0 & 57.0 \\
BNM \cite{bnm_cvpr}& 89.6 & 61.5 & 76.9 & 55.0 & 89.3 & 69.1 & 81.3 & 65.5 & 90.0 & 47.3 & 89.1 & 30.1 & 70.4 \\
CDAN+E \cite{cdanpe_nips}& 85.2 & 66.9 & 83.0 & 50.8 & 84.2 & 74.9 & 88.1 & 74.5 & 83.4 & 76.0 & 81.9 & 38.0 & 73.9 \\
SAFN \cite{safn_iccv19}& 93.6 & 61.3 & 84.1 & 70.6 & 94.1 & 79.0 & 91.8 & 79.6 & 89.9 & 55.6 & 89.0 & 24.4 & 76.1 \\
CGDM \cite{cgdm_cvpr}& 93.7 & 82.7 & 73.2 & 68.4 & 92.9 & 94.5 & 88.7 & 82.1 & 93.4 & 82.5 & 86.8 & 49.2 & 82.3 \\
SHOT \cite{shot_cvpr}& 94.3 & 88.5 & 80.1 & 57.3 & 93.1 & 93.1 & 80.7 & 80.3 & 91.5 & 89.1 & 86.3 & 58.2 & 82.9 \\
CoWA-JMDS \cite{lee2022confidence}& 96.2 & 89.7 & 83.9 & 73.8 & 96.4 & 97.4 & 89.3 & 86.8 & 94.6 & 92.1 & 88.7 & 53.8 & 86.9 \\
FixBi \cite{na2021fixbi}& 96.1 & 87.8 & 90.5 & 90.3 & 96.8 & 95.3 & 92.8 & 88.7 & 97.2 & 94.2 & 90.9 & 25.7 & 87.2 \\
CAN \cite{kang2019contrastive}& 97.0 & 87.2 & 82.5 & 74.3 & 97.8 & 96.2 & 90.8 & 80.7 & 96.6 & 96.3 & 87.5 & 59.9 & 87.2 \\
\midrule
TVT-B \cite{tvt_cvpr}& 92.9 & 85.6 & 77.5 & 60.5 & 93.6 & 98.2 & 89.4 & 76.4 & 93.6 & 92.0 & 91.7 & 55.7 & 83.9 \\
CDTrans-B \cite{cdtrans_iclr}& 97.1 & 90.5 & 82.4 & 77.5 & 96.6 & 96.1 & 93.6 & 88.6 & 97.9 & 86.9 & 90.3 & 62.8 & 88.4 \\
SSRT-B \cite{ssrt_cvpr}& 98.9 & 87.6 & 89.1 & 84.8 & 98.3 & 98.7 & 96.3 & 81.1 & 94.9 & 97.9 & 94.5 & 43.1 & 88.8 \\
DCST \cite{dcst_nn} & 98.1 & 86.4 & 89.2 & 82.5 & 97.7 & 97.6 & 97.8 & 82.5 & 98.0 & 91.7 & 95.6 & 54.8 & 89.3 \\
TSGF-B \cite{tsgf_pr}& 98.9 & 92.9 & 87.7 & 71.2 & 98.4 & 98.6 & 93.1 & 84.3 & 97.3 & 96.1 & 95.4 & 70.0 & 90.3 \\
FFTAT-B \cite{fftat_wacv}& 99.7 & 98.5 & \textbf{93.1} & \textbf{81.1} & 99.8 & \textbf{99.5} & \textbf{97.8} & \textbf{89.6} & 95.7 & 99.8 & 98.7 & 72.4 & 93.8 \\
\midrule
\textbf{JFPD (ours)} & \textbf{100.0} & \textbf{98.8} & 93.0 & 76.2 & \textbf{99.9} & \textbf{99.5} & \textbf{97.8} & 88.5 & \textbf{97.8} & \textbf{99.9} & \textbf{99.2} & \textbf{80.9} & \textbf{94.3} \\
\bottomrule
\end{tabular}
}
\end{table}

\begin{table*}[tbp]
\setlength{\tabcolsep}{0.2em}
\renewcommand{\arraystretch}{0.50}
\centering
\caption{Performance comparison with SOTA methods on DomainNet across all domain shifts.}
\label{tab:domainnet}
\resizebox{\textwidth}{!}{
\begin{tabular}{l|cccccc|c||cccccc|c||cccccc|c||cccccc|c||cccccc|c}
\toprule
 & \multicolumn{7}{c||}{ResNet-101 \cite{he2016deep}} 
 & \multicolumn{7}{c||}{MIMTFL \cite{mimtfl_eccv}} 
 & \multicolumn{7}{c||}{CGDM \cite{cgdm_cvpr}}
 & \multicolumn{7}{c||}{MDD+SCDA \cite{mddpscda_iccv}}
 & \multicolumn{7}{c}{DSAN \cite{dsan_kbs}} \\
\midrule
 & clp & inf & pnt & qdr & rel & skt & Avg. 
 & clp & inf & pnt & qdr & rel & skt & Avg. 
 & clp & inf & pnt & qdr & rel & skt & Avg. 
 & clp & inf & pnt & qdr & rel & skt & Avg. 
 & clp & inf & pnt & qdr & rel & skt & Avg. \\
\midrule
clp & - & 19.3 & 37.5 & 11.1 & 52.2 & 41.1 & 32.2 
    & - & 15.1 & 35.6 & 10.7 & 51.5 & 43.1 & 31.2
    & - & 16.9 & 35.3 & 10.8 & 53.5 & 36.9 & 30.7
    & - & 20.4 & 43.3 & 15.2 & 59.3 & 46.5 & 36.9
    & - & 19.2 & 46.8 & 14.9 & 65.6 & 47.5 & 38.8 \\
inf & 30.2 & - & 31.2 & 3.6 & 44.0 & 27.9 & 27.4
    & 32.1 & - & 31.0 & 2.9 & 48.5 & 31.0 & 29.1
    & 27.8 & - & 28.2 & 4.4 & 48.2 & 22.5 & 26.2
    & 32.7 & - & 34.5 & 6.3 & 47.6 & 29.2 & 30.1
    & 44.6 & - & 40.7 & 10.9 & 56.8 & 39.4 & 38.5 \\
pnt & 39.6 & 18.7 & - & 4.9 & 54.5 & 36.3 & 30.8
    & 40.1 & 14.7 & - & 4.2 & 55.4 & 36.8 & 30.2
    & 37.7 & 14.5 & - & 4.6 & 59.4 & 33.5 & 30.0
    & 46.4 & 19.9 & - & 8.1 & 58.8 & 42.9 & 35.2
    & 53.7 & 18.7 & - & \textbf{12.0} & 67.1 & 45.5 & 39.4 \\
qdr & 7.0 & 0.9 & 1.4 & - & 4.1 & 8.3 & 4.3
    & 18.8 & 3.1 & 5.0 & - & 16.0 & 13.8 & 11.3
    & 14.9 & 1.5 & 6.2 & - & 10.9 & 10.2 & 8.7
    & 31.1 & 6.6 & 18.0 & - & 28.8 & 22.0 & 21.3
    & 34.3 & 11.1 & 27.3 & - & 34.5 & 29.6 & 27.4 \\
rel & 48.4 & 22.2 & 49.4 & 6.4 & - & 38.8 & 33.0
    & 48.5 & 19.0 & 47.6 & 5.8 & - & 39.4 & 22.1
    & 49.4 & 20.8 & 47.2 & 4.8 & - & 38.2 & 32.0
    & 55.5 & 23.7 & 52.9 & 9.5 & - & 45.2 & 37.4
    & 53.7 & 20.0 & 49.3 & 12.2 & - & 45.9 & 36.2 \\
skt & 46.9 & 15.4 & 37.0 & 10.9 & 47.0 & - & 31.4
    & 51.7 & 16.5 & 40.3 & 12.3 & 53.5 & - & 34.9
    & 50.1 & 16.5 & 43.7 & 11.1 & 55.6 & - & 35.4
    & 55.8 & 20.1 & 46.5 & 15.0 & 56.7 & - & 38.8
    & 54.9 & 19.2 & 47.2 & 14.4 & 66.7 & - & 40.5 \\
\midrule
Avg. & 34.4 & 15.3 & 31.3 & 7.4 & 40.4 & 30.5 & 26.6
     & 38.2 & 13.7 & 31.9 & 7.2 & 45.0 & 32.8 & 28.1
     & 36.0 & 14.0 & 32.1 & 7.1 & 45.5 & 28.3 & 27.2
     & 44.3 & 18.1 & 39.0 & 10.8 & 50.2 & 37.2 & 33.3
     & 48.2 & 17.6 & 42.3 & 12.9 & 58.1 & 41.6 & 36.8 \\
\midrule
 & \multicolumn{7}{c||}{ViT-B \cite{dosovitskiy2020image}}
 & \multicolumn{7}{c||}{SSRT-B \cite{ssrt_cvpr}} 
 & \multicolumn{7}{c||}{FFTAT-B \cite{fftat_wacv}} 
 & \multicolumn{7}{c||}{DSAN-B \cite{dsan_kbs}} 
 & \multicolumn{7}{c}{\textbf{JFPD (Ours)}} \\
\midrule
 & clp & inf & pnt & qdr & rel & skt & Avg. 
 & clp & inf & pnt & qdr & rel & skt & Avg. 
 & clp & inf & pnt & qdr & rel & skt & Avg. 
 & clp & inf & pnt & qdr & rel & skt & Avg. 
 & clp & inf & pnt & qdr & rel & skt & Avg. \\
\midrule
clp & - & 27.2 & 53.1 & 13.2 & 71.2 & 53.3 & 43.6
    & - & 33.8 & 60.2 & 19.4 & 75.8 & 59.8 & 49.8
    & - & \textbf{39.4} & \textbf{70.3} & \textbf{25.5} & 81.9 & 70.9 & \textbf{57.6}
    & - & 38.0 & 66.8 & 13.2 & 84.4 & \textbf{71.2} & 54.7
    & - & 38.0 & 66.9 & 21.1 & \textbf{87.6} & 65.0 & 55.7 \\
inf & 51.4 & - & 49.3 & 4.0 & 66.3 & 41.1 & 42.4
    & 55.5 & - & 54.0 & 9.0 & 68.2 & 44.7 & 46.3
    & \textbf{67.4} & - & 65.9 & \textbf{12.6} & 79.4 & 60.0 & 57.1
    & 65.7 & - & 65.8 & 9.8 & 84.0 & 61.6 & 57.4
    & 66.1 & - & \textbf{66.5} & 8.5 & \textbf{84.2} & \textbf{62.0} & \textbf{57.5} \\
pnt & 53.1 & 25.6 & - & 4.8 & 70.0 & 41.8 & 39.1
    & 61.7 & 28.5 & - & 8.4 & 71.4 & 55.2 & 45.0
    & \textbf{71.9} & \textbf{37.9} & - & 11.4 & 81.5 & 65.6 & \textbf{53.6}
    & 68.7 & 35.2 & - & 5.6 & \textbf{85.4} & \textbf{66.3} & 52.2
    & 69.1 & \textbf{37.9} & - & 9.5 & 81.5 & 65.6 & 52.7 \\
qdr & 30.5 & 4.5 & 16.0 & - & 27.0 & 19.3 & 19.5
    & 42.5 & 8.8 & 24.2 & - & 37.6 & 33.6 & 29.3
    & 43.5 & 12.4 & 28.9 & - & 41.5 & 32.3 & 31.7
    & 50.5 & \textbf{21.2} & \textbf{35.2} & - & 38.6 & \textbf{39.2} & \textbf{36.9}
    & \textbf{51.1} & 20.4 & 32.9 & - & \textbf{41.5} & 38.7 & \textbf{36.9} \\
rel & 58.4 & 29.0 & 60.0 & 6.0 & - & 45.8 & 39.9
    & 69.9 & 37.1 & 66.0 & 10.1 & - & 58.9 & 48.4
    & \textbf{77.7} & \textbf{37.2} & 74.3 & \textbf{14.2} & - & 64.2 & 53.5
    & 75.3 & 36.5 & 74.7 & 9.0 & - & \textbf{66.5} & 52.4
    & 75.8 & \textbf{37.2} & \textbf{75.3} & 8.1 & - & \textbf{66.5} & \textbf{52.6} \\
skt & 63.9 & 23.8 & 52.3 & 14.4 & 67.4 & - & 44.4
    & 70.6 & 32.8 & 62.2 & 21.7 & 73.2 & - & 52.1
    & 78.4 & 34.9 & 70.1 & \textbf{28.4} & 78.1 & - & 58.0
    & \textbf{81.3} & \textbf{36.3} & \textbf{72.6} & 14.9 & 85.3 & - & 58.1
    & 80.9 & 34.3 & 72.5 & 17.8 & \textbf{85.5} & - & \textbf{58.2} \\
\midrule
Avg. & 51.5 & 22.0 & 46.1 & 8.5 & 60.4 & 40.3 & 38.1
     & 60.0 & 28.2 & 53.3 & 13.7 & 65.3 & 50.4 & 45.2
     & 67.8 & 32.4 & 61.9 & \textbf{18.4} & 72.5 & 58.6 & 51.9
     & 68.3 & 33.4 & \textbf{63.0} & 10.5 & 75.5 & \textbf{61.0} & 52.0
     & \textbf{68.6} & \textbf{33.6} & 62.8 & 13.0 & \textbf{76.1} & 59.6 & \textbf{52.3} \\
\bottomrule
\end{tabular}
}
\end{table*}

\section{Experiment}

\subsection{Experimental Setup}
\label{sec:exp-setup}

\textbf{Datasets.}
We evaluate the proposed method on four widely used domain adaptation benchmarks.
\textit{(i) Digits} includes four datasets: MNIST~\cite{726791}, SVHN~\cite{netzer2011reading}, USPS~\cite{hull1994database}, and Synthetic Digits~\cite{roy2018effects}. 
Each dataset is used once as the source domain while the remaining datasets serve as target domains, resulting in 12 transfer tasks. 
We follow the standard evaluation protocol using the official train/test splits.
\textit{(ii) Office-Home}~\cite{venkateswara2017deep} contains four domains: Art, Clipart, Product, and Real-World, spanning 65 shared object categories with approximately 15.5K images in total.
\textit{(iii) VisDA-2017}~\cite{peng2017visda} is a large-scale synthetic-to-real object classification benchmark consisting of more than 280K synthetic images for the source domain and 55K real images for the target domain across 12 categories.
\textit{(iv) DomainNet}~\cite{peng2019moment} is a large-scale benchmark with 345 object categories across six domains: Clipart, Real, Sketch, Infograph, Painting, and Quickdraw. 
The dataset contains over 560K images. Following the standard protocol, we evaluate all source-target domain pairs across these domains.

\textbf{Setups.} 
On \textit{Digits}, our objective is to analyze architectural behavior under domain shift rather than to benchmark against prior methods. 
We evaluate a VGG-style CNN and a lightweight Vision Transformer (ViT-S). The CNN comprises three convolutional layers (32, 64, 128 channels) followed by two fully connected layers (256 and 10 units) with ReLU activations. The ViT-S uses a patch size of 4, 9 transformer layers, 28 attention heads, a 420-dimensional embedding, and an MLP head of size 88 with GELU activation and dropout rate 0.06.
Models are pretrained on each source domain for 60 epochs. The CNN uses batch size 64, learning rate $10^{-3}$, and weight decay $10^{-5}$, while ViT-S uses batch size 32, learning rate $10^{-4}$, and weight decay $2\times10^{-2}$ with cosine restarts every 22 epochs. Adaptation is performed for 30 epochs on unlabeled target data with learning rates $10^{-4}$ (CNN) and $5\times10^{-5}$ (ViT-S) under cosine scheduling.

For standard benchmarks, we follow established UDA protocols to ensure fair and direct comparison with prior work. We use ViT-B/16 models pretrained on ImageNet-1K, with classification heads adapted to the number of target classes.
\textit{Office-Home:} Models are pretrained on the source domain for 100 epochs (batch size 128, learning rate $10^{-4}$, weight decay $5\!\times\!10^{-5}$, cosine restarts every 40 epochs), followed by 50 epochs of adaptation on all unlabeled target samples (batch size 128, learning rate $5\!\times\!10^{-5}$, cosine restarts every 20 epochs).
\textit{VisDA-2017:} Source training is performed for 30 epochs (batch size 128, learning rate $10^{-3}$, cosine annealing), followed by 100 epochs of adaptation on unlabeled target data (batch size 128, learning rate $10^{-4}$).
\textit{DomainNet:} Models are trained following standard protocols, with source pretraining followed by adaptation on all unlabeled target samples using stochastic optimization and cosine scheduling.
All methods are evaluated with comparable backbone architectures and training protocols while preserving their original components. Baseline results are obtained from official implementations or prior work when available. Hyperparameters are selected using source validation or standard settings, without access to target labels. 
For JFPD, target batch sizes range from 128 to 512, and source prototypes are estimated from class-balanced subsets with 32 samples per class (relying only on source supervision).  

\textbf{Baselines.}
We compare the proposed method with several representative domain adaptation approaches spanning different methodological categories.
For all methods, the same backbone architectures and training protocols are used to ensure fair comparison.
\textit{(i) Feature-alignment methods.} We compare against classical feature-level domain adaptation approaches that learn domain-invariant representations by aligning feature distributions, such as DANN~\cite{ganin2016domain}. 
\textit{(ii) Adversarial domain adaptation.}
We include adversarial alignment methods that learn domain-invariant features through domain discriminators, including 
DANN~\cite{ganin2016domain} and CDAN~\cite{cdanpe_nips}. 
\textit{(iii) Prediction-aware and confidence-based adaptation.}
To evaluate reliability-aware adaptation strategies, we compare with methods that use prediction confidence or entropy during adaptation, including 
DUET~\cite{duet_nips} and 
CoWA-JMDS~\cite{lee2022confidence}. 
\textit{(iv) Strong recent baselines.}
We further compare against recent high-performing methods that achieve strong results on large-scale benchmarks such as DomainNet, including FFTAT-B \cite{fftat_wacv}, DSAN~\cite{dsan_kbs} and 
DCST~\cite{dcst_nn}. 
\textit{(v) Ablation baselines.}
To analyze the contributions of individual components in the proposed framework, we additionally evaluate two variants: 
\textit{FGPD}, which uses only feature-guided discrepancy, and 
\textit{PGFD}, which uses only prediction-guided discrepancy. 
The full model is denoted as \textit{JFPD}.

\subsection{Comparisons with the State of the Art}


\textbf{Overall performance.} JFPD consistently achieves the best results across all benchmarks. On Office-Home (Table \ref{tab:officehome_full}), JFPD attains an average accuracy of 92.2\%, outperforming the strongest prior method FFTAT-B (91.4\%) and substantially surpassing earlier methods such as CDAN+E (68.5\%) and SHOT (71.8\%). The gains are consistent across most transfer tasks, including challenging shifts such as Ar$\rightarrow$Cl and Pr$\rightarrow$Cl, indicating robust cross-domain generalization.
On VisDA-2017 (Table \ref{tab:visda_comparison}), JFPD achieves a new state of the art with an average accuracy of 94.3\%, improving over FFTAT-B (93.8\%) and TSGF-B (90.3\%). Notable improvements are observed in difficult categories such as \emph{truck} and \emph{plant}, where domain discrepancies are more pronounced, demonstrating the effectiveness of trust-aware discrepancy modeling in large-scale synthetic-to-real adaptation.
On DomainNet (Table \ref{tab:domainnet}), JFPD achieves the best overall performance with an average accuracy of 52.3\%, outperforming strong baselines such as DSAN-B (52.0\%) and FFTAT (51.9\%). The improvements are consistent across diverse domain pairs, particularly for visually complex domains (\eg, \emph{real} and \emph{sketch}), highlighting robustness under large domain variability.

\textbf{Comparison across method categories.} Compared with feature-alignment methods (\eg, DANN), JFPD yields substantial improvements, indicating that aligning feature distributions alone is insufficient under significant domain shifts. By jointly modeling feature and prediction discrepancies, JFPD captures both structural and semantic mismatch.
Relative to prediction-aware and confidence-based methods (\eg, CoWA-JMDS), JFPD consistently improves performance. While these methods exploit prediction confidence, they typically treat feature alignment implicitly. In contrast, JFPD explicitly integrates prediction uncertainty with feature alignment via cross-guided trust, leading to more reliable discrepancy estimation.
Compared with recent transformer-based methods (\eg, FFTAT-B, DSAN-B, and DCST), JFPD still achieves superior results under comparable backbones. This suggests that the gains arise from the proposed trust-aware joint discrepancy formulation rather than architectural differences.

\begin{figure*}[tbp]
  \centering
  \includegraphics[width=\linewidth]{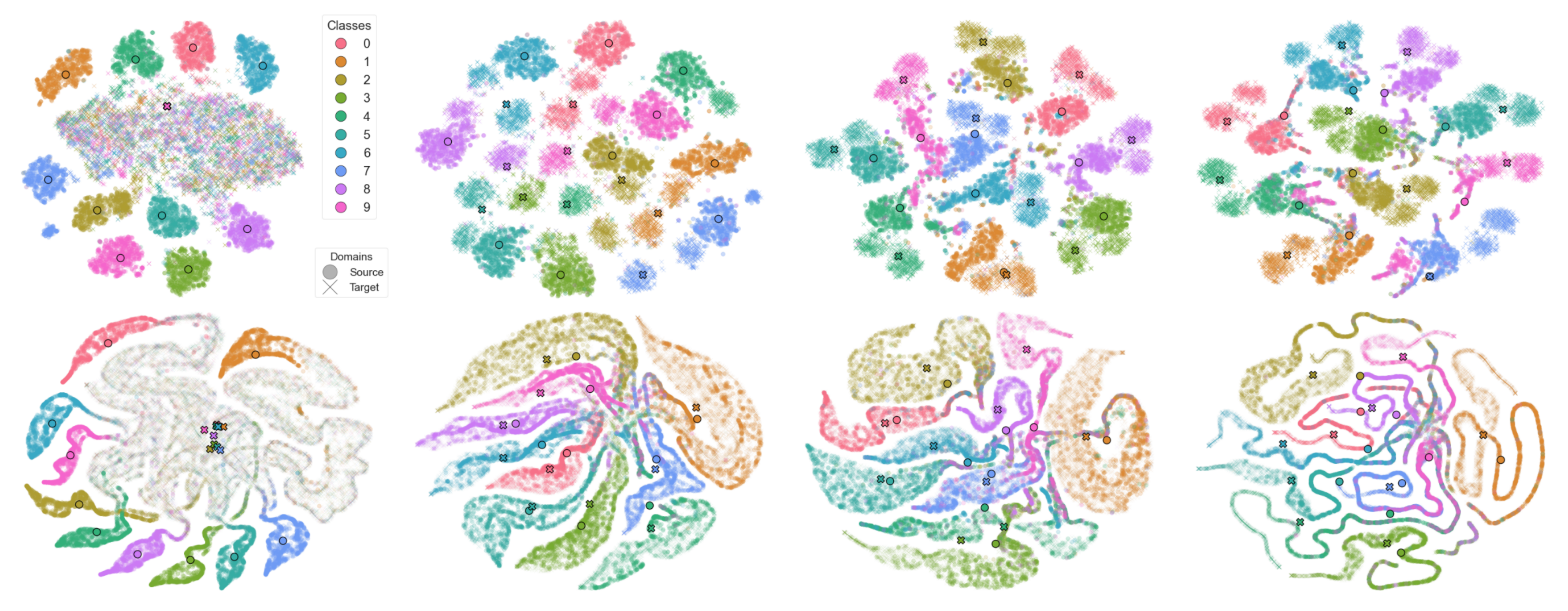}
  \caption{t-SNE visualizations of MNIST (source, dots) and SVHN (target, crosses). 
  All subplots share the same legend.
  The first row shows learned feature representations; the second row shows softmax predictions. From left to right: (i) pretrained MNIST model (no SVHN fine-tuning), (ii) standard fine-tuning, (iii) our feature-guided prediction divergence, and (iv) our full Joint Feature-Prediction Discrepancy (JFPD) method.
  In the feature space (top row), our method (columns iii-iv) produces the most compact, class-consistent clusters across domains, indicating strong structural alignment. In the prediction space (bottom row), JFPD yields clearer decision boundaries, improved semantic consistency, and better confidence calibration, key for reliable cross-domain prediction.
  These results highlight the benefit of jointly aligning features and predictions with trust-aware guidance, especially for semantically ambiguous or uncertain classes. Zoom in for a clearer view.
  }
  \label{fig:comp}
\end{figure*}

\subsection{Analysis and Discussion}

\textbf{t-SNE visualization analysis.} Fig.~\ref{fig:comp} visualizes the feature (top row) and prediction (bottom row) spaces for MNIST (source) and SVHN (target). 
In the feature space, the pretrained model shows severe source-target misalignment with fragmented, overlapping clusters. Standard fine-tuning reduces this gap slightly, but clusters remain loosely formed. FGPD improves alignment by incorporating prediction signals, yet lacks consistent class separation. JFPD yields the most compact, class-coherent clusters across domains, with well-separated class boundaries, indicating strong structural alignment and improved cross-domain representation learning.
In the prediction space, JFPD again stands out by producing clearer, more semantically aligned decision boundaries and fewer misclassifications. Predictions are confident and consistent, even for ambiguous classes (\eg, 3, 4, 9), whereas competing methods exhibit fuzzy boundaries and scattered predictions.
These results demonstrate that JFPD’s joint alignment of features and predictions, guided by trust, enables robust domain adaptation. 

\textbf{Backbone performance.} We evaluate JFPD on both lightweight and high-capacity models: a VGG-style CNN and ViT-S on Digits (Table~\ref{tab:vit_cnn_digits_all}).
JFPD consistently outperforms standard fine-tuning across all settings, with large gains on CNN. For example, CNN model achieves gains of +19.25\%, +12.35\%, and +9.55\% for MNIST, SVHN, and USPS to Syn, respectively. While ViT-S shows some improvement, its gains are smaller, likely because ViTs are better suited for tasks involving complex visual details, such as when domains contain background clutter (like in SVHN). Additionally, the relatively simple nature of Digits means they don't fully leverage the strengths of ViTs.
Notably, PGFD underperforms relative to FGPD and JFPD, especially with ViTs. This suggests that relying solely on feature-level alignment without prediction-level grounding (\eg, class semantics) can lead to suboptimal adaptation in expressive models.

\textbf{Joint feature-prediction perspective.}
Fig.~\ref{fig:domain-gap-error} visualizes the relationship between JFPD and the resulting target-domain error across different domain shifts on Digits. Each point represents a source-to-target adaptation pair, and points sharing the same color correspond to the same target domain.

Two important observations emerge. First, a clear positive correlation exists between mean JFPD and target error, indicating that larger feature-prediction discrepancies correspond to more difficult adaptation scenarios. This behavior confirms that JFPD effectively captures the magnitude of domain shift and can serve as a reliable indicator of adaptation difficulty.
Second, the relationship becomes substantially more structured after applying JFPD-based training (Fig.~\ref{subfig2}) compared with standard fine-tuning (Fig.~\ref{subfig1}). Under standard fine-tuning, the discrepancy values are scattered and less predictive of target performance. In contrast, the trust-aware JFPD objective produces a clearer alignment between discrepancy magnitude and target error, demonstrating that the joint feature-prediction formulation provides a more meaningful characterization of cross-domain mismatch.
This behavior is consistent with the theoretical properties of the proposed formulation. As shown in Lemma~\ref{lemma:trust_suppression}, the trust-aware weighting suppresses unreliable target samples and yields conservative yet informative estimates of domain discrepancy. By jointly modeling representation divergence and prediction disagreement while modulating their influence through trust scores, JFPD provides a stable and interpretable measure of domain shift.

\textbf{Domain-specific observations.} Fig.~\ref{fig:domain-gap-error} also shows consistent patterns across different target domains. Transfers toward more complex or noisy domains (\eg, Synthetic Digits) tend to exhibit higher JFPD values and correspondingly larger target errors, reflecting the increased difficulty of aligning representations under significant appearance shifts. Conversely, transfers between visually similar domains (\eg, MNIST and USPS) yield smaller discrepancies and lower errors.
These observations highlight the practical value of JFPD as a diagnostic tool for understanding cross-domain adaptation behavior. 

\textbf{Dataset transfer difficulty.} Digits datasets differ in visual style and noise characteristics. Transfers involving large style discrepancies (\eg, MNIST $\rightarrow$ SVHN or USPS $\rightarrow$ Synthetic) therefore present the greatest adaptation challenges (Fig.~\ref{fig:domain-gap-error}). In these scenarios, JFPD improves performance by jointly aligning representation structure and prediction semantics while suppressing unreliable correspondences through trust-aware weighting. As a result, the proposed framework consistently reduces adaptation difficulty across challenging domain shifts.

\begin{figure}[tbp]
\centering

\subfloat[Standard fine-tuning]{
\includegraphics[width=0.485\linewidth]{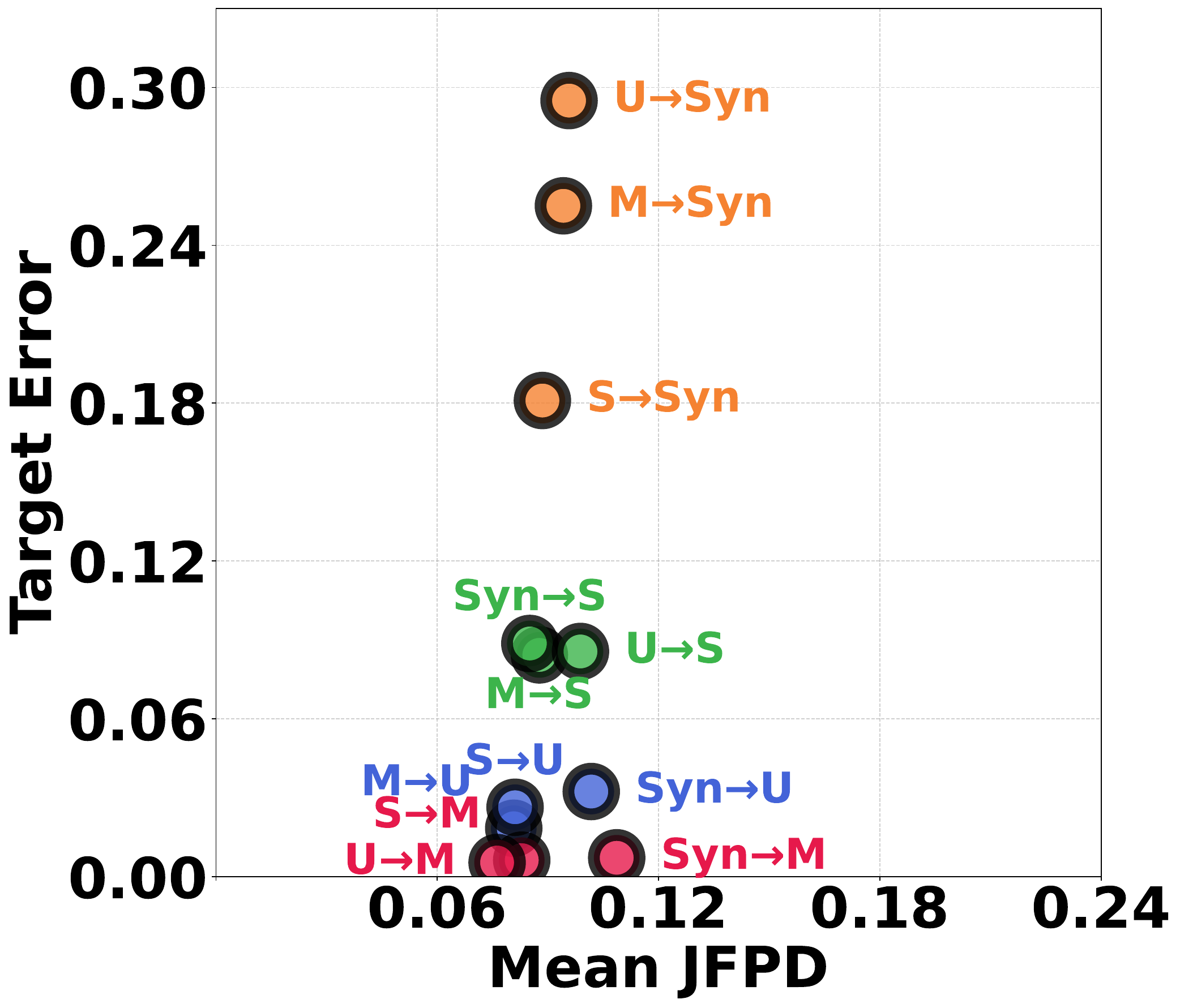}\label{subfig1}}
\subfloat[Our JFPD fine-tuning]{
\includegraphics[width=0.485\linewidth]{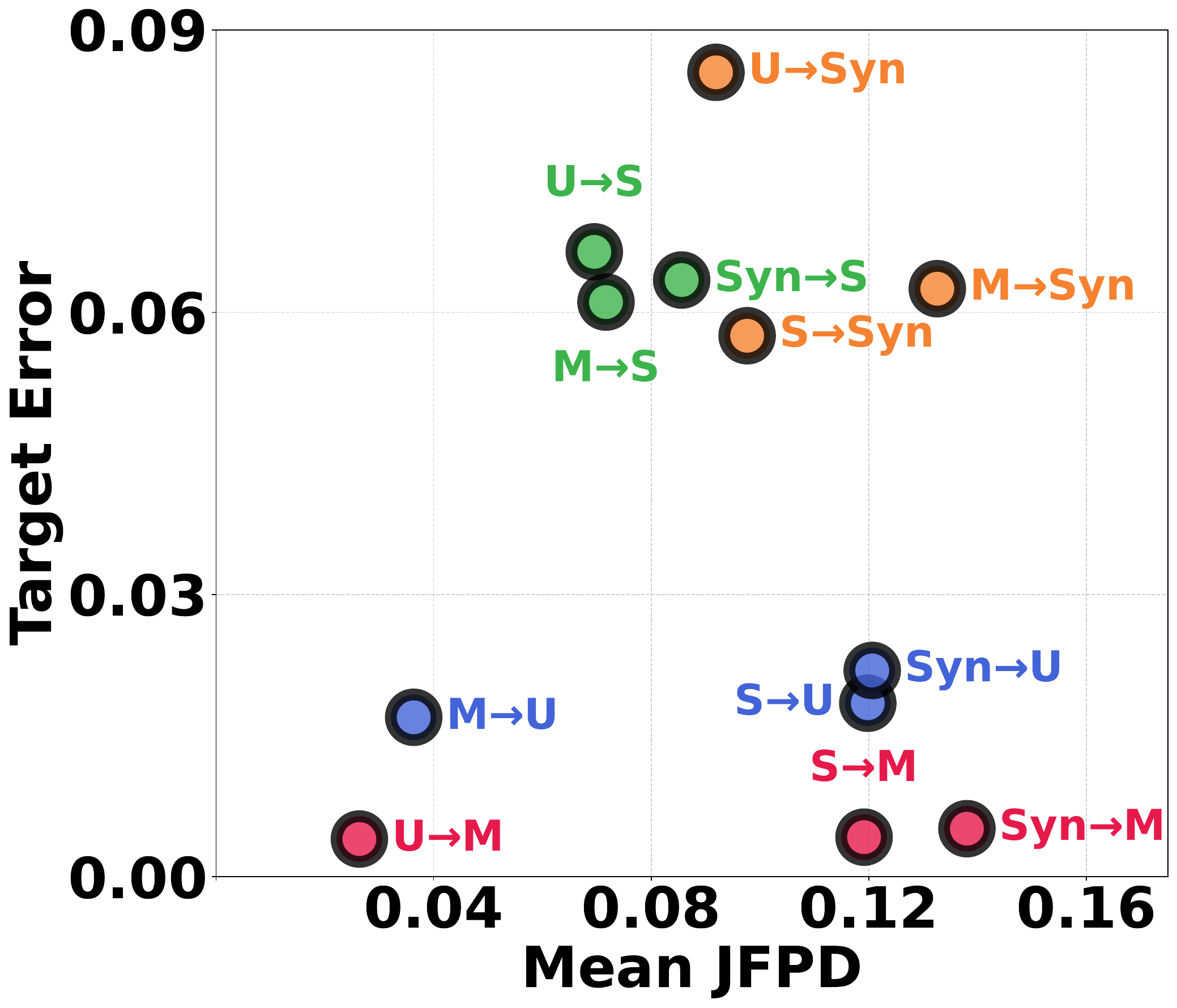}\label{subfig2}}
\caption{Mean JFPD \vs~ target error across domain shifts on Digits. 
In each plot, points with the same color share the same target domain.
Each point denotes a source-to-target adaptation. The positive correlation shows \textbf{JFPD's sensitivity to domain shifts}: higher JFPD values align with greater target error, highlighting its potential as \textit{a diagnostic tool} for assessing adaptation difficulty beyond standard fine-tuning.}
\label{fig:domain-gap-error}
\end{figure}

\begin{figure}[tbp]
\centering
\subfloat[Gap \vs epochs.\label{fig:gap}]{
\includegraphics[width=0.49\linewidth]{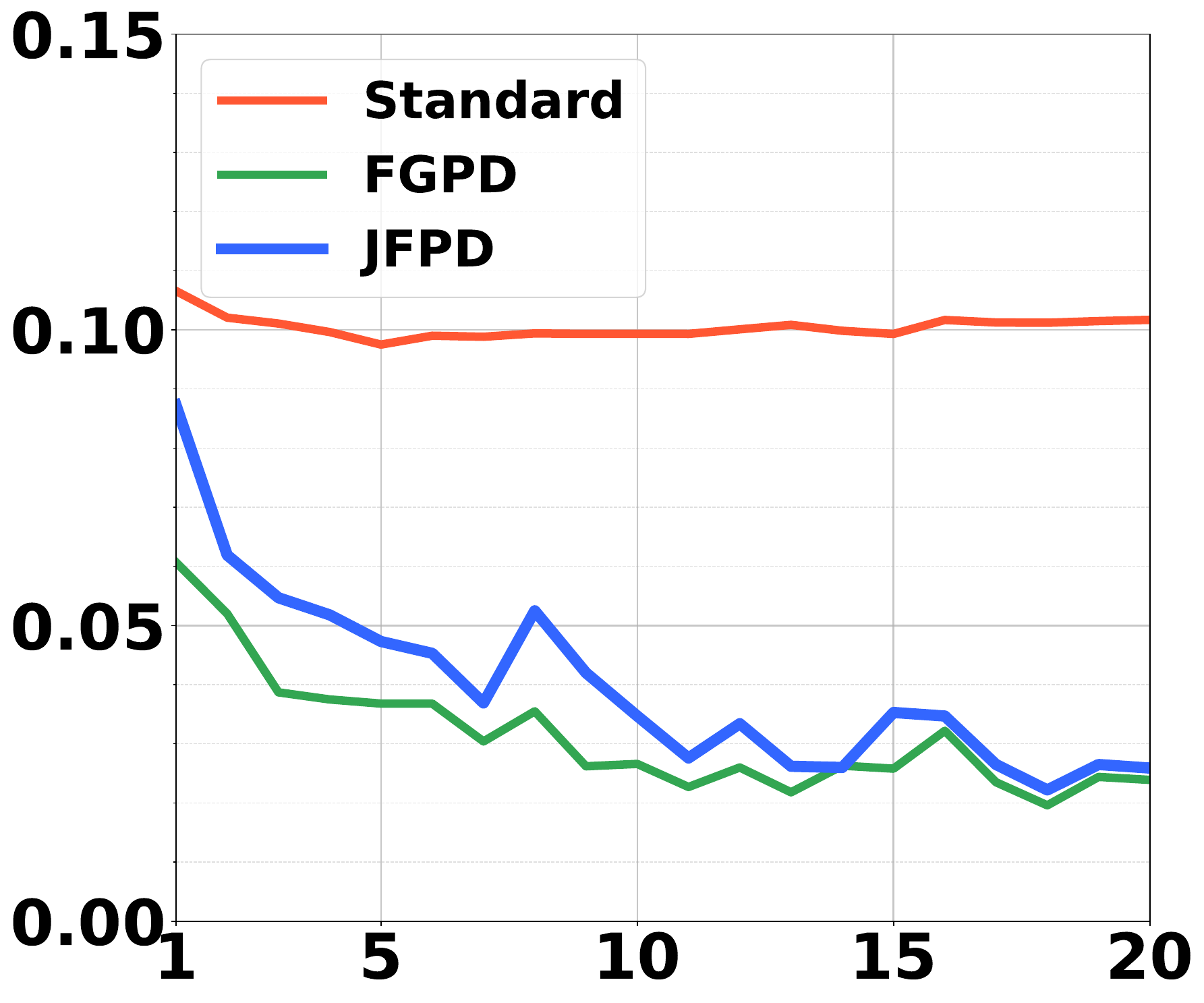}}
\subfloat[$\alpha$ sensitivity.\label{hyper:alpha}]{
\includegraphics[width=0.49\linewidth]{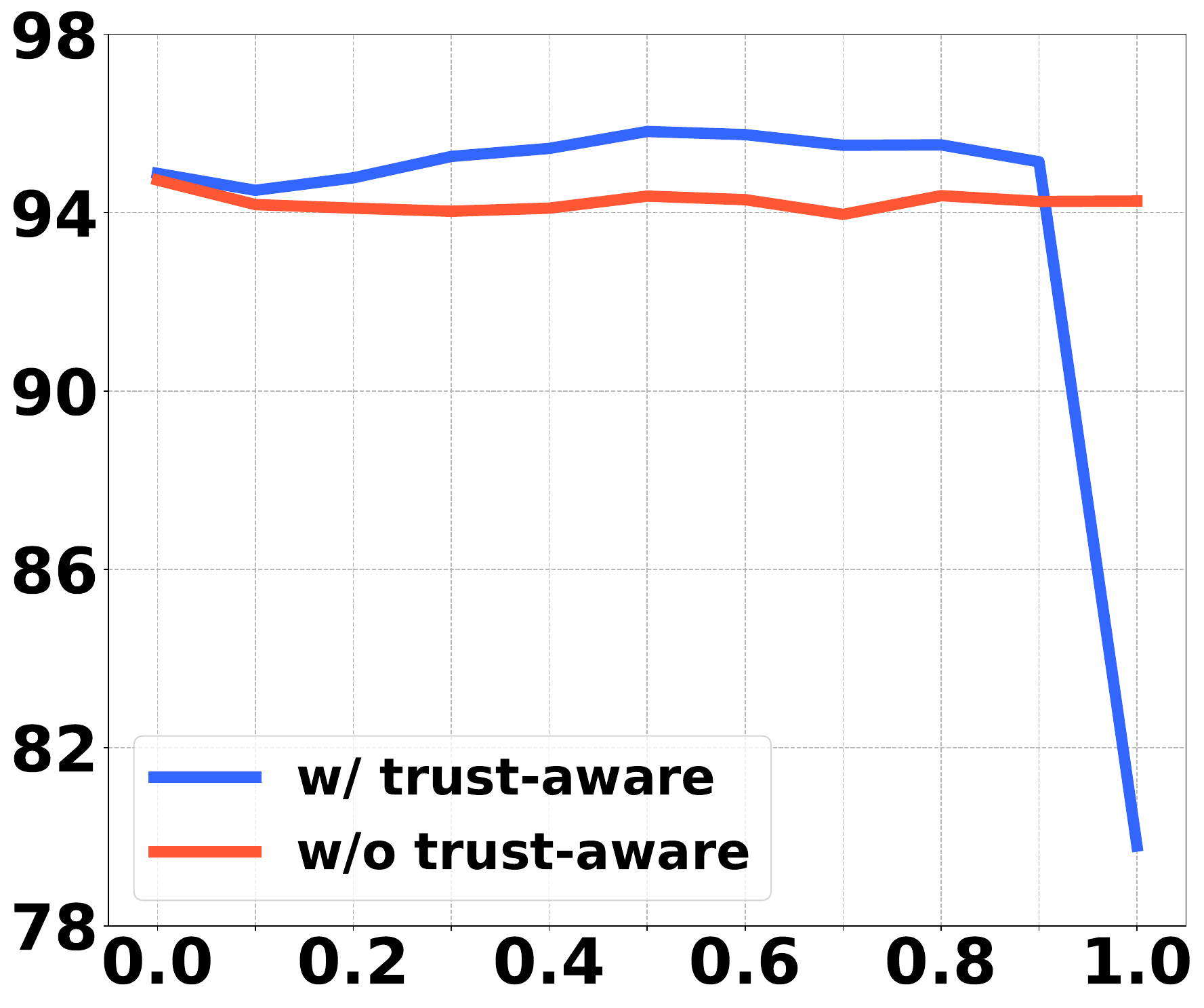}}
\caption{
Behavior of the proposed JFPD during adaptation.
(a) Domain discrepancy across training epochs. While standard fine-tuning shows minimal reduction of the domain gap, both FGPD and JFPD progressively align source and target distributions. JFPD achieves the lowest discrepancy throughout training, indicating that joint feature-prediction alignment with trust-aware weighting leads to more stable adaptation.
(b) Sensitivity to the balance parameter $\alpha$. Performance remains stable across a wide range of $\alpha$ values, demonstrating the robustness of JFPD. The sharp degradation at $\alpha=1.0$ highlights the limitation of prediction-only alignment and confirms the importance of balancing feature and prediction discrepancies.
}
\label{fig:ablation}
\end{figure}

\subsection{Ablation Study}

\textbf{Trust-aware gap reduction.}
Fig.~\ref{fig:gap} shows the evolution of the domain discrepancy during training for standard fine-tuning, FGPD, and the proposed JFPD. Standard fine-tuning shows almost no reduction in the domain gap, indicating that cross-entropy optimization alone is insufficient to align source and target distributions. In contrast, both FGPD and JFPD progressively reduce the discrepancy over training epochs, demonstrating that explicitly modeling domain divergence provides a stronger adaptation signal.
Among the compared methods, JFPD consistently achieves the lowest gap throughout training. This improvement arises from the joint modeling of feature-level and prediction-level discrepancies, which allows the adaptation process to capture both representation alignment and semantic consistency. Moreover, the trust-aware weighting mechanism further stabilizes the optimization by emphasizing reliable samples and suppressing noisy or uncertain ones, leading to \textit{more consistent gap reduction over time}.

\textbf{Effect of discrepancy balance.}
Fig.~\ref{hyper:alpha} evaluates the sensitivity of JFPD to the discrepancy balance parameter $\alpha$ on the MNIST$\rightarrow$SVHN transfer task. The parameter $\alpha$ controls the relative importance of feature discrepancy and prediction discrepancy.
At $\alpha=0$, the objective reduces to feature-guided alignment (FGPD). In this case, the model relies solely on feature similarity, achieving an accuracy of $94.9\%$. While feature alignment provides useful structural information, it does not fully capture semantic consistency between domains.
At the opposite extreme, $\alpha=1$, the objective relies entirely on prediction discrepancy (PGFD). This setting leads to a dramatic performance drop to $79.8\%$, indicating that prediction-only alignment is unstable when target predictions are unreliable. Without feature-level constraints, prediction divergence alone can amplify noisy target signals.
Between these extremes, performance remains consistently strong for $\alpha \in [0.1,0.9]$, reaching the best performance around $\alpha=0.5$-$0.6$ ($95.8\%$). This behavior demonstrates that feature-level and prediction-level discrepancies provide complementary signals for domain alignment. Balancing the two enables the model to simultaneously capture representation similarity and semantic agreement, leading to more robust adaptation.

\textbf{Role of trust-aware weighting.} To further analyze the impact of the trust-aware design, we evaluate JFPD \textit{without the trust factors} $\phi$ and $\psi$, which reduces the formulation to \textit{a simple unweighted combination of feature and prediction discrepancies}. Under this setting, performance follows a similar trend as $\alpha$ varies from $0$ to $1$, but remains consistently lower (with the exception of $\alpha \!=\! 1$) across the evaluated $\alpha$ values.
Although joint discrepancies still provide reasonable alignment signals, the absence of trust-aware weighting limits the model's ability to distinguish reliable from unreliable target samples. By contrast, incorporating the uncertainty-aware and semantic-alignment trust factors allows JFPD to emphasize confident and semantically consistent samples during adaptation. This reliability-aware mechanism reduces the influence of noisy correspondences and leads to more stable optimization and improved target performance.


\section{Conclusion}

We presented \emph{Joint Feature-Prediction Discrepancy (JFPD)}, a trust-aware domain adaptation framework that jointly integrates feature-level and prediction-level divergences, modulated by sample-wise trust. Unlike prior single-view or uniform weighting approaches, JFPD explicitly models the interplay between representations and classifier outputs, emphasizing reliable target samples through complementary uncertainty- and alignment-based trust estimates.
Our framework introduces a principled discrepancy metric and a corresponding trust-aware adaptation objective, enabling stable and selective alignment of target samples with source prototypes. Experiments on digit and object recognition benchmarks show that JFPD not only improves adaptation accuracy but also provides interpretable insights into how reliable samples drive cross-domain alignment.
Future work includes extending JFPD to multi-source, continual, and cross-modal domain adaptation scenarios, and integrating trust-aware alignment with large-scale pretrained models, including vision-language architectures, where calibrated reliability estimates are critical for robust transfer.


\bibliographystyle{IEEEtran}
\bibliography{research}

\vskip -1\baselineskip plus -1fil

\begin{IEEEbiography}[{\includegraphics[trim=0 120 0 50, width=1in,height=1.25in,clip,keepaspectratio]{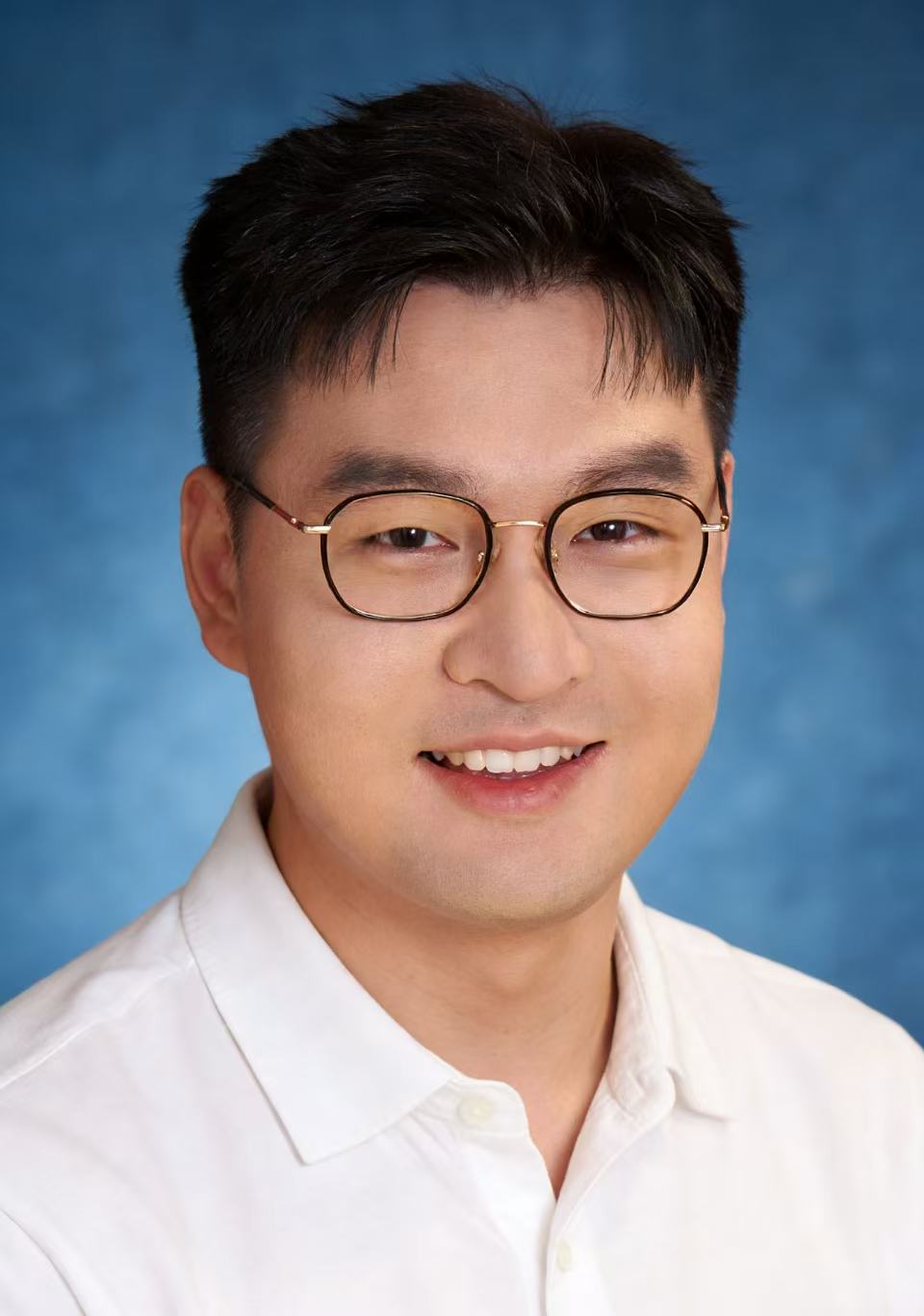}}]{Xi Ding} received the M.S. degree in machine learning from the Australian National University (ANU) in 2025. He is currently a research intern at Carnegie Mellon University (CMU) and has previously conducted research at the Australian Research Council (ARC) Research Hub and the Temporal Intelligence and Motion Extraction (TIME) Lab. His research spans machine learning, computer vision, large language models, kernel and tensor methods, and domain adaptation. He has published in top-tier venues, including NeurIPS, ICLR, and AAAI, and actively contributes to the research community as a reviewer and coordinator.
\end{IEEEbiography}

\vskip -1\baselineskip plus -1fil

\begin{IEEEbiography}
[{\includegraphics[width=1in,height=1.25in,clip,keepaspectratio]{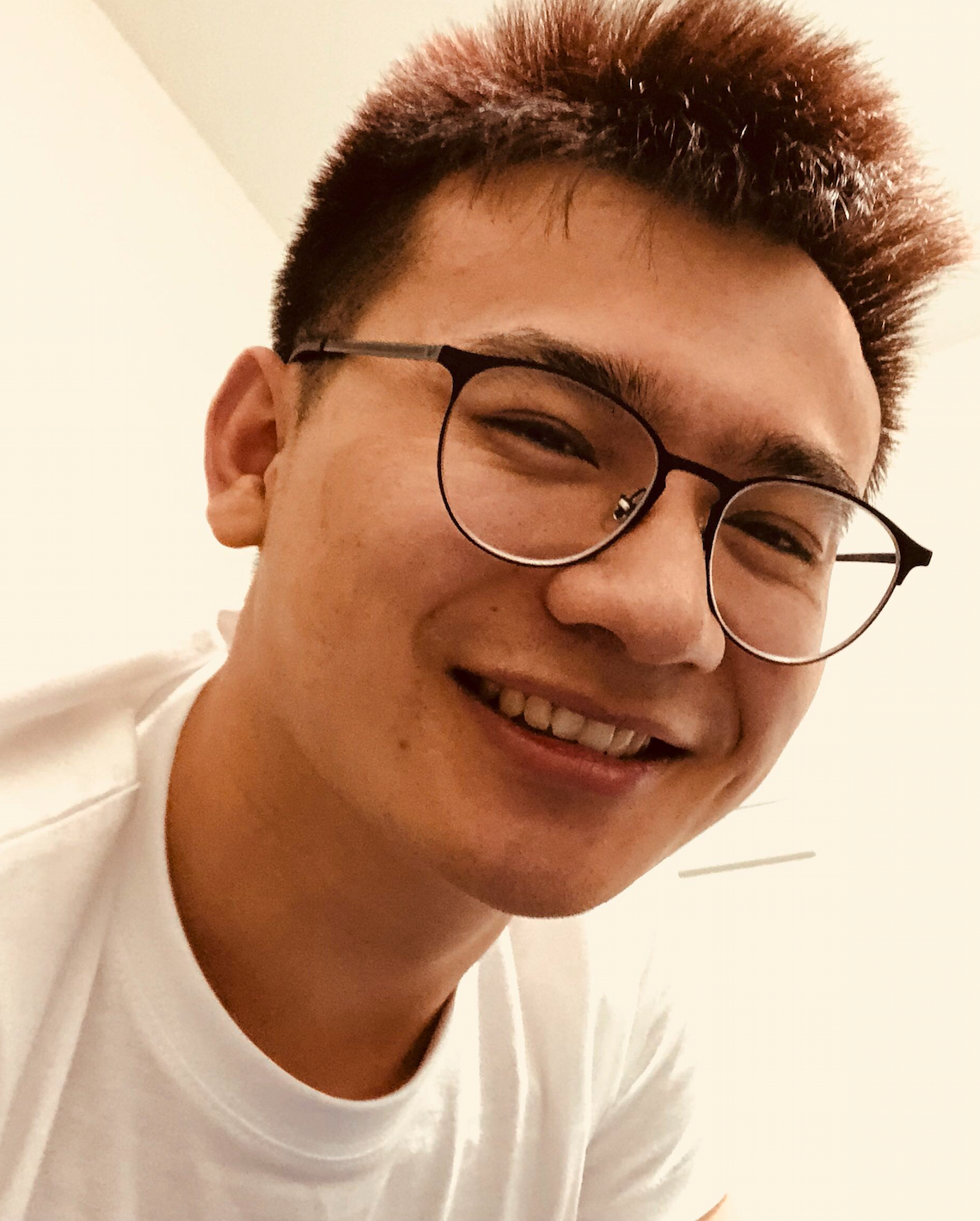}}]{Lei Wang} received his M.E. in Software Engineering from the University of Western Australia (UWA) in 2018 and his Ph.D. in Engineering and Computer Science from the Australian National University (ANU) in 2023. He is a Research Fellow in the School of Electrical and Electronic Engineering at Griffith University and a Visiting Scientist with Data61/CSIRO. He leads the Temporal Intelligence and Motion Extraction (TIME) Lab at Griffith University. He previously held research positions at ANU, UWA, and Data61/CSIRO. His research focuses on motion-, data-, and model-centric approaches to video action recognition and anomaly detection. He has authored numerous first-author papers in top-tier venues, including CVPR, ICCV, ECCV, ACM Multimedia, NeurIPS, ICLR, ICML, AAAI, TPAMI, IJCV, and TIP, and received the Sang Uk Lee Best Student Paper Award at ACCV 2022. He serves as an Area Chair for ACM Multimedia 2024-2025, ICASSP 2025, and ICPR 2024, and was recognized as an Outstanding Area Chair at ACM Multimedia 2024.
\end{IEEEbiography}

\vskip -1\baselineskip plus -1fil

\begin{IEEEbiography}
[{\includegraphics[trim=0 0 0 300, width=1in,height=1.25in,clip,keepaspectratio]{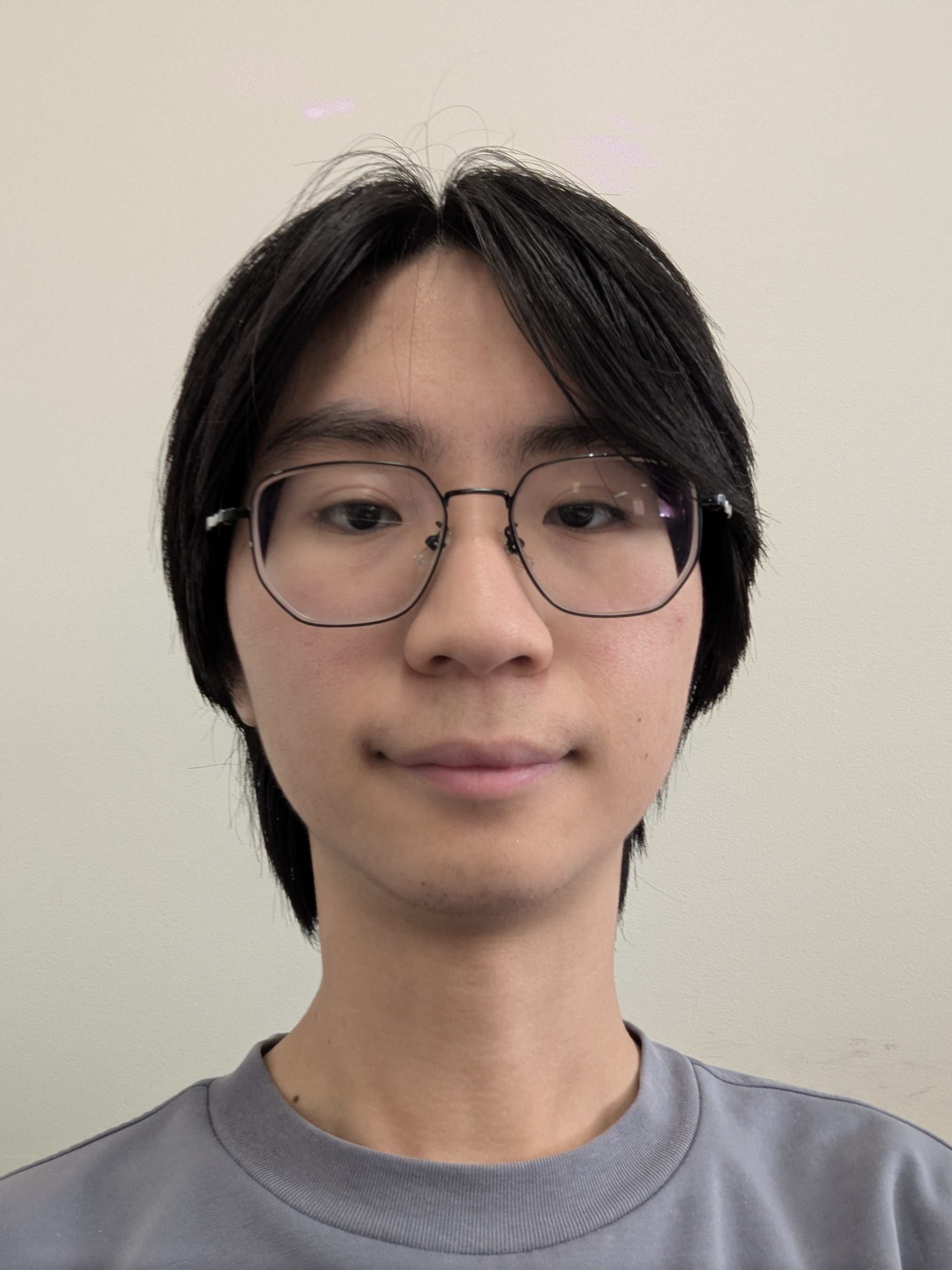}} 
]{Syuan-Hao Li} received the B.S. degree in Computer Science, National Taitung University (NTTU), Taiwan, in 2025. He is currently a Ph.D. pathway student at Griffith University and a research intern at the Temporal Intelligence and Motion Extraction (TIME) Lab. He serves as a workshop coordinator for TIME 2026: the 2nd International Workshop on Transformative Insights in Multi-faceted Evaluation, hosted at the Web Conference (WWW 2026). His research interests include temporal modeling, multimodal intelligence, and fine- and ultra-fine-grained visual understanding.
\end{IEEEbiography}

\vskip -1\baselineskip plus -1fil

\begin{IEEEbiography}[{\includegraphics[trim=0 0 0 0, width=1in,height=1.25in,clip,keepaspectratio]{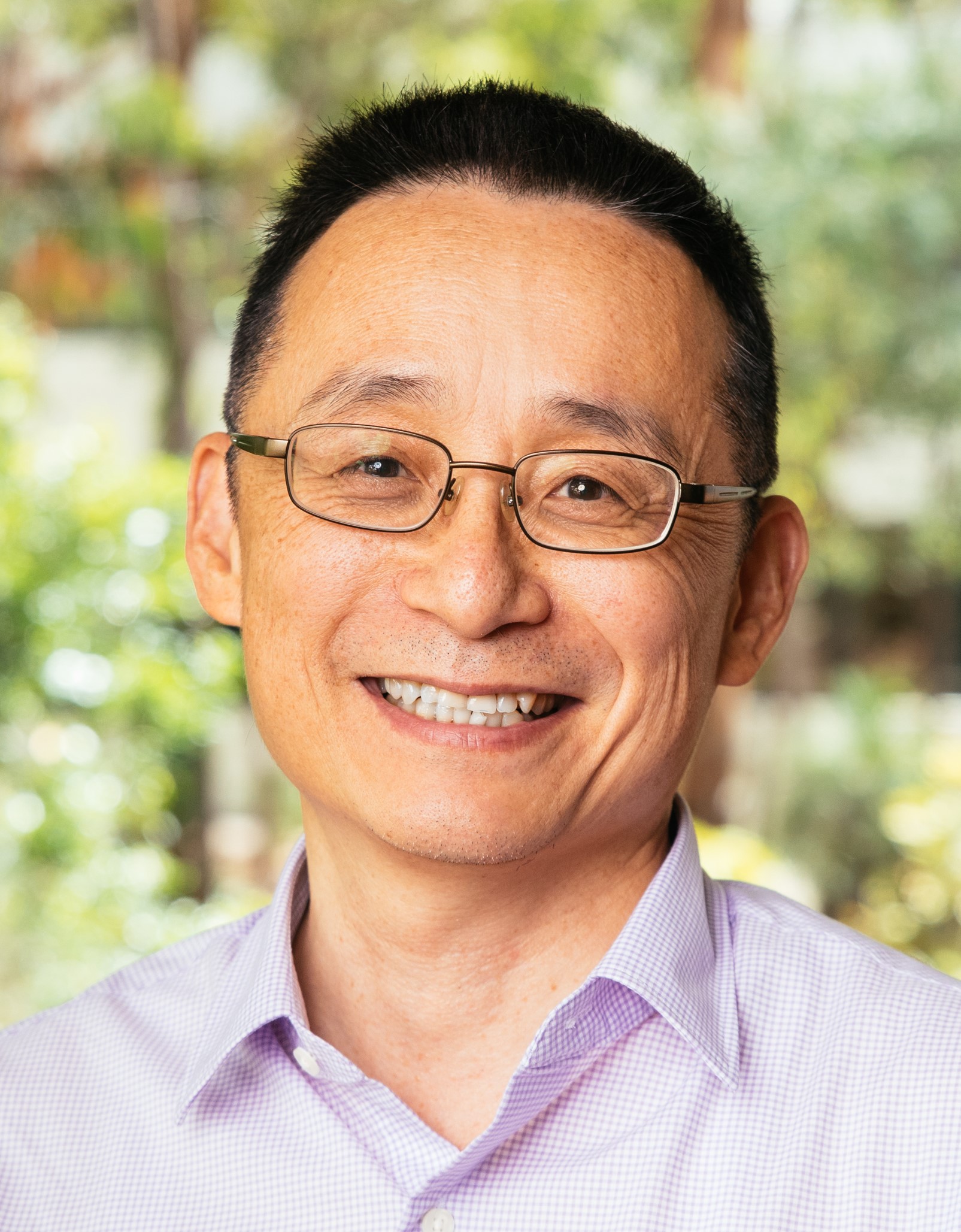}}
]{Yongsheng Gao} received the BSc and MSc degrees in Electronic Engineering from Zhejiang University, China, in 1985 and 1988, respectively, and the PhD degree in Computer Engineering from Nanyang Technological University, Singapore. He is currently a Professor with the School of Engineering and Built Environment, Griffith University, and Director of the ARC Research Hub for Driving Farming Productivity and Disease Prevention, Australia. He was previously the Leader of the Biosecurity Group at the Queensland Research Laboratory, National ICT Australia (ARC Centre of Excellence), a consultant at Panasonic Singapore Laboratories, and an Assistant Professor at Nanyang Technological University. His research interests include smart farming, machine vision for agriculture, biosecurity, face recognition, biometrics, image retrieval, computer vision, pattern recognition, environmental informatics, and medical imaging. He is a recipient of the 2025 ARC Industry Laureate Fellow.\end{IEEEbiography}



\end{document}